\documentclass{article}
\usepackage{log_2025}						

\usepackage[export]{adjustbox}
\usepackage{booktabs}						
\usepackage{multirow}						
\usepackage{amsfonts}						
\usepackage{graphicx}						
\usepackage{duckuments}						

\usepackage{wrapfig}

\usepackage[british]{babel}
\usepackage[belowskip=-5pt]{caption}
\usepackage[aboveskip=-5pt]{subcaption}
\usepackage{amsthm}
\usepackage{amssymb}
\usepackage{mathtools}
\usepackage{svg}

\usepackage{algorithm}
\usepackage{algpseudocode}

\usepackage[table]{xcolor}
\usepackage{colortbl}
\usepackage{tikz}
\usetikzlibrary{positioning}
\usetikzlibrary{calc}

\usepackage[numbers,compress,sort]{natbib}	

\newtheorem{definition}{Definition}
\newtheoremstyle{theorem}
  {4pt}
{4pt}
{}
{}
{\bfseries}
{.}
{.5em}
{}
\theoremstyle{theorem}
\newtheorem{thm}{Theorem}
\newtheorem{lem}[thm]{Lemma}
\newtheorem{corollary}[thm]{Corollary}

\DeclareMathOperator{\spn}{span}
\DeclareMathOperator{\Part}{partition}
\DeclareMathOperator{\vol}{vol}

\title[An Improved and Generalised Analysis for Spectral Clustering]{An Improved and Generalised Analysis for Spectral Clustering}

\author[G. Tyler \& L. Zanetti]{%
George Tyler\\
Department of Mathematical Sciences \\
University of Bath \\
\email{grlt20@bath.ac.uk}\And
Luca Zanetti \\
Department of Mathematical Sciences \\
University of Bath \\
\email{lz2040@bath.ac.uk}
}

\begin{document}
\newenvironment{sproof}{%
 \renewcommand{\proofname}{Sketch Proof}\proof}{\endproof}\textbf{}

\maketitle

\begin{abstract}
We revisit the theoretical performances of Spectral Clustering, a classical algorithm for graph partitioning that relies on the eigenvectors of a matrix representation of the graph. Informally, we show that Spectral Clustering works well as long as the smallest eigenvalues appear in groups well separated from the rest of the matrix representation's spectrum. 
This arises, for example, whenever there exists a hierarchy of clusters at different scales, a regime not captured by previous analyses.
Our results are very general and can be applied beyond the traditional graph Laplacian. In particular, we study Hermitian representations of digraphs and show Spectral Clustering can recover partitions where edges between clusters are oriented mostly in the same direction. This has applications in, for example, the analysis of trophic levels in ecological networks. We demonstrate that our results accurately predict the performances of Spectral Clustering on synthetic and real-world data sets.
\end{abstract}

\section{Introduction}

Spectral Clustering \citep{ng2001spectral} is one of the most popular algorithms for clustering a graph. It exploits the eigenvectors of a matrix representing the graph to compute a low-dimensional Euclidean embedding of the vertices, which is then partitioned using geometric clustering algorithms, such as $k$-means.

While Spectral Clustering has been shown to work well in practice in a wide variety of settings, it still lacks a complete theoretical understanding. Typically, analyses of Spectral Clustering work by showing that the chosen eigenvectors are close to linear combinations of the indicator vectors of the clusters. For example, the structure theorem of \citet{peng2015partitioning}  informally says that the bottom $k$ eigenvectors of the graph Laplacian are close to linear combinations of the $k$ ``best'' clusters  whenever two conditions are satisfied: (1) the $k$-way expansion constant\footnote{The $k$-way expansion constant is a measure of quality of a partition~\citep{lee2014multiway} similar to the normalised cut \citep{shimalik}.} of the optimal partition is \emph{small}; (2) the $(k+1)$-st smallest eigenvalue of the Laplacian is \emph{large}. However, the first condition is often not satisfied in practice and, we argue, is actually not necessary for Spectral Clustering to perform well.

Our main contribution is a strengthening of the structure theorem: we show Spectral Clustering works well whenever the informative eigenvalues are well-separated from the rest of the spectrum. This not only allows us to recover  known bounds achieved by either the structure theorem or by classical perturbation arguments such as the Davis-Kahan  theorem \citep{davis1970rotation}, but also to obtain new improved bounds in many scenarios. In particular, our techniques are well-suited to deal with situations where there exists a hierarchy of clusters at different scales. We believe our results are the first to correctly predict the excellent performances of Spectral Clustering for a large class of real-world networks.

An additional advantage of our analysis is that it applies to any Hermitian positive semidefinite representation of a graph, beyond the standard graph Laplacian. We showcase this strength by analysing recently proposed Hermitian representations of digraphs \citep{cucuringu2020hermitian}. In particular, we consider a clustering problem  in which we aim to partition the vertices of a digraph into subsets $S_1,\dots,S_k$ so that most of the edges are from $S_i$ to $S_{i+1}$ for any $i \in \{1,\dots,k-1\}$. In other words, we would like most of the edges to follow a directed path between the clusters. This has applications, for example, in uncovering trophic levels in food chains or detecting patterns in trade networks \citep{laenen2020higher}. 
We provide a new cost function for this task and apply our structure theorem to obtain bounds that accurately predict the performances of Spectral Clustering on both synthetic and real-world data sets.

\subsection{Related work}
There is a wealth of literature on Spectral Clustering; in this section we discuss only the most relevant work on the subject and we refer the reader to the classical surveys \cite{fortunatosurvey,vonlux} for additional background. 

In the case of random graphs, and specifically stochastic block models, analyses of Spectral Clustering typically rely on the Davis-Kahan theorem \citep{davis1970rotation} or similar tools \citep{rohe11}. Classical perturbation arguments, however, are poorly suited to deal with non-random graphs in which ``noise'' might be localised in relatively small regions of the graph~\citep{phdzanetti}. 
For this reason, \citet{peng2015partitioning}  introduced their structure theorem, which states that the bottom $k$-eigenvectors of the normalised Laplacian of an undirected graph are close to the indicator vectors of the ``optimal'' $k$ clusters whenever the ratio $k^2 \cdot \rho(k)/\lambda_{k+1}$ is small, where $\rho(k)$ is the $k$-way expansion constant \citep{lee2014multiway}, and $\lambda_{k+1}$ is the $(k+1)$-th smallest eigenvalue of the Laplacian. A series of recent results \cite{kolev,mizutani}, culminating with \citep{macgregor2022tighter}, further simplified the proof of the structure theorem, while improving the dependency on the number of clusters $k$.

All of these results, however, require $\rho(k)$ to be small: this is often not satisfied in real-world networks nor in stochastic block models, even for a range of parameters where Spectral Clustering is known to work well. 
Consider for example a graph consisting of two cliques of 50 vertices connected by a perfect matching in which each edge has weight $20$. Spectral clustering perfectly partitions the graph by dividing the two cliques. If we apply the structure theorem of \citet{macgregor2022tighter}, however, we obtain a bound on the distance between the eigenvector of the Laplacian associated with $\lambda_2$ and the closest linear combination of indicator vectors of the clusters equals to $0.28$. This is because, even though there exists a large gap in the Laplacian spectrum, the value of $\rho(2)$ is relatively large ($\approx 0.28$). In contrast, our Corollary~\ref{cor:RemoveFirstEvec} will show that this distance is actually zero, correctly predicting the exact recovery of the clusters by spectral clustering.

Furthermore, our analysis is not restricted to the traditional (normalised) graph Laplacian and undirected graphs, but it can handle any Hermitian and positive semidefinite representation of undirected or directed graphs. Hermitian representations of digraphs were recently investigated for clustering purposes by \citet{cucuringu2020hermitian}, who proposed their use to recover clusters characterised by a strong imbalance in the direction of the inter-cluster edges. They define a directed stochastic block model that captures this problem and show  Spectral Clustering is able to recover its communities. \citet{laenen2020higher} consider the specific instance of this clustering task in which the direction of most edges is required to follow a directed path between the clusters. They apply the techniques of \citet{peng2015partitioning} to obtain a structure theorem and an analysis of Spectral Clustering for appositely-constructed Hermitian representations of digraphs. In our work, we revisit this task and argue that the results of Laenen and Sun do not capture the practical performances of Spectral Clustering. Indeed, we demonstrate simple examples where Spectral Clustering works very well, but Laenen and Sun's results are completely uninformative. Applying our improved structure theorem together with a new cost function allows us to prove bounds that correctly predict Spectral Clustering's practical performances, vastly outperforming Laenen and Sun's results.

\subsection{Organisation}
The paper is organised as follows: in Section~\ref{sec:background}, we cover the necessary background. In Section~\ref{sec:general}, we state our main result, which is a generalised and improved structure theorem. We show how our results allow us to obtain meaningful bounds for Spectral Clustering, even for graphs where the $k$-way expansion constant is relatively large. In Section~\ref{sec:digraphs}, we apply our generalised structure theorem to analyse Spectral Clustering on Hermitian Laplacians for digraphs. Proofs, together with additional experiments, are included in the Appendix. Code is available in a GitHub repository\footnote{\url{https://github.com/GeorgeRLTyler/Improved-and-Generalised-Analysis-for-Spectral-Clustering}}.

\section{Background}
\label{sec:background}
Let \(\mathcal{G} = (V,E,w)\) be a (possibly directed) graph with \(N\) vertices , \(M\) edges and weight function \(w: V \times V \rightarrow \mathbb{R}_{\geq 0}\). For any edge \(e = (u,v) \in E\), we write the weight of \(e\) as \(w_{uv}\) or \(w_e\). For an undirected graph \(\mathcal{G}\), we denote the degree of vertex \(u\) by \(d(u) = \sum_{v \in V} w_{uv}\). If \(\mathcal{G}\) is directed, we define the in-degree and out-degree of vertex \(u\) as \(d_{\text{in}}(u) = \sum_{v \in V} w_{vu}\) and \(d_{\text{out}}(u) = \sum_{v \in V} w_{uv}\). The degree of vertex \(u\) in this case is \(d(u) = d_{\text{in}}(u) + d_{\text{out}}(u)\). 
For any two sets \(S,T \in V\), we define the cut value \(w(S,T) = \sum_{\substack{(u,v) \in E \\ u \in S, v \in T}} w_{uv}\) and the volume of \(S \subset V\) as \(\text{vol}(S) = \sum_{u \in S} d(u)\). 

We define the conductance of \(\emptyset \ne S \subset V\) as
\[
\Phi(S) = \frac{w(S,V-S)}{\text{vol}(S)}.
\]
A \(k\)-way partition of \(V\) is a collection of \(k\) subsets \(S_1, \hdots, S_k \subset V\) where \(S_i \cap S_j = \emptyset\) for \(i \neq j\) and \(\bigcup_{i=1}^k S_i = V\).
For undirected graphs, we measure the quality of a partition by the \(k\)-way expansion:
\[
\Phi(S_1,\dots,S_k) = \max_{1 \leq i \leq k} \Phi(S_i).
\]
The \(k\)-way expansion constant of $\mathcal{G}$ \cite{lee2014multiway} is defined as
\[
\rho(k) = \min_{\text{partition } S_1, \hdots S_k} \Phi(S_1,\dots,S_k).
\]

Spectral clustering leverages eigenvalues and eigenvectors of Hermitian matrices associated with the graph. The adjacency matrix \(A \in \mathbb{R}^{N \times N}\) of an undirected graph \(\mathcal{G} = (V,E,w)\) is defined as $A_{uv} = w_{uv}$ if $(u,v) \in E$ and zero otherwise.
When \(\mathcal{G}\) is directed, we will use the Hermitian adjacency matrix as defined in \cite{cucuringu2020hermitian}:
\begin{equation}
A_{uv} = \begin{cases}
    w_{uv}\exp\left(2 \pi \mathrm{i}/\tilde{k}\right) & \text{ if } u \rightarrow v,  \\
    w_{uv}\exp\left(-2 \pi \mathrm{i}/\tilde{k}\right) & \text{ if } u \leftarrow v, \label{eq:adjdig}\\
    0 & \text{ otherwise}
\end{cases}
\end{equation}
where $\mathrm{i}$ is the imaginary unit and the value \(\Tilde{k}\) must be prescribed. Throughout our paper we will assume for simplicity that, for a directed graph, $(u,v) \in E \implies (v,u) \not\in E$.
The following definitions all apply to both undirected and directed graphs using each respective adjacency matrix.

The degree matrix \(D \in \mathbb{R}^{N \times N}\) is a diagonal matrix where each diagonal entry equals the degree of a vertex: $D_{uu} = d(u)$.
The Laplacian matrix \(L \in \mathbb{R}^{N \times N}\) is defined as $L = D - A$, while the normalized Laplacian matrix \(\mathcal{L} \in \mathbb{R}^{N \times N}\) is defined as
$
\mathcal{L} = D^{-1/2} L D^{-1/2} = I - D^{-1/2} A D^{-1/2}$.

We will use \(M \in \mathbb{C}^{N \times N}\) to denote any 
 Hermitian positive semidefinite matrix representation of a graph. We will denote its eigenvectors by \(f_1, \hdots, f_N \in \mathbb{C}^N\) with corresponding eigenvalues \( \lambda_1 \leq \hdots \leq \lambda_N \). We will write the eigendecomposition of $M$ as $M = F \Delta F^*$ with $F=(f_1 \, f_2 \, \cdots \, f_N)$ and $\Delta = \text{diag}(\lambda_1, \hdots, \lambda_N)$.


Given a basis of orthonormal vectors \(g_1, \hdots, g_N \in \mathbb{C}^N\), we denote their Rayleigh quotients by \(\gamma_i = g_i^* M g_i\), with \(\gamma_1 \leq \hdots \leq \gamma_N\). We assemble these vectors in a matrix \(G \in \mathbb{C}^{N \times N}\).

There are many variants of Spectral Clustering in the literature. Our results are quite general and will apply to most of these variants. For simplicity, we will consider the variant defined in Algorithm~\ref{alg:spectral}, which is the one considered in \cite{peng2015partitioning}. For undirected graphs, we typically choose $\tilde{k} = k$ and $M = \mathcal{L}$, which has real eigenvectors. The third step of Algorithm~\ref{alg:spectral} is optional and typically depends on the matrix representation used; for example, it is needed when $M$ is the normalised Laplacian $\mathcal{L}$, but not when $M$ is the combinatorial Laplacian $L$.
\vspace{-0.1cm}
\begin{algorithm}[h!]
\begin{algorithmic}[1]
    \State \textbf{Input:} $\mathcal{G}=(V,E,w),M \in \mathbb{C}^{|V| \times |V|},k \ge 2,\tilde{k}\le k$
    \State Compute the matrix $F \in \mathbb{C}^{|V| \times \tilde{k}}$ whose columns are the orthonormal eigenvectors $f_1,\dots,f_{\tilde{k}}$ associated to $\lambda_1 \le \cdots \le \lambda_{\tilde{k}}$.
    \State $\tilde{F} \gets D^{1/2} F$
    \State Minimise the following $k$-means objective:
    \[
        \min_{c_1,\dots,c_k \in \mathbb{C}^{\tilde{k}}} \sum_{u \in V} d(u) \|\tilde{F}_{u,\colon} - c_i \|_2^2
    \]
    \State \textbf{Output:} partition $A_1,\dots,A_k$ corresponding to the solution of the above $k$-means problem.
\end{algorithmic}
\caption{Spectral Clustering}
\label{alg:spectral}
\end{algorithm}

\section{A generalised Structure Theorem} \label{sec:general}

Our first main result is an improved and generalised structure theorem. It shows that, for any matrix $M$, as long as $\lambda_{k+1} \gg \lambda_1$, any set of $k$ orthonormal vectors with Rayleigh quotient close to the smallest eigenvalue of $M$ must be close to linear combinations of the bottom $k$ eigenvectors of $M$. Furthermore, the bottom $k$ eigenvectors of $M$ will be close to linear combinations of these vectors. 
\pagebreak

\begin{thm} \label{thm:general}
    Let $M \in \mathbb{C}^{N \times N}$ be Hermitian and positive semidefinite with eigenvalues $0 \le \lambda_1 \le \cdots \le \lambda_N$ and corresponding orthonormal basis of eigenvectors $f_1,\dots,f_N$. Let $g_1,\dots,g_k \in \mathbb{C}^N$ be orthonormal and let $\gamma_i \coloneq g_i^* M g_i$ $(1 \le i \le k < N)$.  Then, if $\lambda_{k+1} > \lambda_1 $,there exist 
    \(\hat{g}_1, \hdots, \hat{g}_k \in \spn\{g_1, \hdots, g_k\}\),
    such that  \vspace{-0.4cm}
    \[ 
    \sum_{i=1}^k \|f_i - \hat{g}_i\|^2 \leq \frac{\sum_{i=1}^k \gamma_i - k \lambda_1}{\lambda_{k+1} - \lambda_1}.
    \] \vspace{-0.2cm}
\end{thm}

We observe we can simply recover the structure theorem - Theorem 1 of \citet{macgregor2022tighter} - by choosing  $M = \mathcal{L}$, the normalised Laplacian of an undirected graph $\mathcal{G} = (V,E,w)$, and letting $g_1,\dots,g_k$ be the (normalised) indicator vectors of the clusters achieving the minimum $\rho(k)$.

\begin{corollary} \label{cor:structure}
    Let \(\mathcal{G}\) be undirected and connected with normalized Laplacian \(\mathcal{L}\).
    Let \(\{S_i\}_{i=1}^k\) be any optimal k-way partition that achieves \(\rho(k)\). For any $1 \le i \le k$, define $\chi_i \in \mathbb{R}^N$ as $\chi_i(u) = 1$ if $u \in S_i$ and $\chi_i(u) = 0$ otherwise. Let \(g_i = \frac{D^{1/2} \chi_i}{\|D^{1/2} \chi_i\|}\). Then, There exist $\hat{g}_{1}, \ldots, \hat{g}_{k} \in \spn\{g_1, \ldots, g_k\}$, such that \vspace{-0.2cm}
    \[
    \sum_{i=1}^{k} \| f_i - \hat{g}_{i} \|^2 \leq \frac{k \rho(k)}{\lambda_{k+1}}.
    \]
\end{corollary}

Informally, Theorem~\ref{thm:general} states that if we choose a matrix representation $M$ of a graph and $g_1,\dots,g_k$ indicator vectors of some clusters, then the bottom $k$ eigenvectors of $M$ must be close to linear combinations of these indicator vectors as long as the Rayleigh quotients of the indicator vectors are significantly smaller than $\lambda_{k+1}$. This is crucial to show that Spectral Clustering works well. This is crucial to show that Spectral Clustering works well, as evidenced by the next lemma, which is adapted from \cite{cucuringu2020hermitian}.

\begin{lem}
\label{lem:kmeans}
Assume the partition $A_1,\dots,A_k$ output by Algorithm~\ref{alg:spectral} is computed by a $(1+\alpha)$-approximation algorithm for $k$-means, with $\tilde{k} = k$. Let $\{S_1, \hdots, S_k\}$ be any $k$-way partition of $V$ achieving $\rho(k)$. Let $G \in \mathbb{C}^{|V| \times \tilde{k}}$ be such that $(D^{-1/2}G)_{u,\colon} = (D^{-1/2}G)_{v,\colon}$ if $u,v\in S_i$ for some $i$. Let $F,\tilde{F}$ be defined as in Algorithm~\ref{alg:spectral}.
For any $i=1,\dots,k$, let $\mu_i = \vol(S_i)^{-1} \cdot \sum_{u \in S_i} d(u) \tilde{F}_{u,\colon}$. Define $\mathcal{D} \coloneq \min_{i,j} \|\mu_i - \mu_j\|$ and $U \coloneq \sum_{u \in V}\|F_{u,\colon}- G_{u,\colon}\|^2$. Assume $U \le (1/5) \mathcal{D}^{-1} (2+\alpha)^{-1} \cdot \min_{i=1,\dots,k} \vol(S_i)$.
Then, the volume of the symmetric difference between $\{S_1,\dots,S_k\}$ and $\{A_1,\dots,A_k\}$ (up to a permutation of the indices) is at most $\mathcal{O}(\frac{(1+\alpha) U}{\mathcal{D}^2})$.
\end{lem}

Notice that, by choosing $ G_{:,i} = \hat{g}_i$ for $i=1,\dots,\tilde{k}$, we have $U = \sum_{i=1}^{\tilde{k}} \|f_i - \hat{g}_i\|^2$. Thus $\mathcal{D}$ is instead dependent on the choice of the indicator vectors. In the case of undirected graph clustering with $M = \mathcal{L}$ and the traditional choice of indicator vectors, as long as $U \ll 1$ (e.g., by Corollary \ref{cor:structure}, whenever $k \rho(k) \ll \lambda_{k+1}$), we have that $\mathcal{D} = \Omega\left(\min_{i=1,\dots,k} \vol(S_i)^{-1}\right)$\citep{macgregor2022tighter}. Assuming a constant-factor approximation algorithm for $k$-means is used, Lemma \ref{lem:kmeans} guarantees that the symmetric difference between the partition output by spectral clustering and the optimal partition is small.

Theorem~\ref{thm:general} together with Lemma~\ref{lem:kmeans} give us a computable way to certify the quality of the partition $A_1,\dots,A_k$ output by spectral clustering compared to the optimal partition $S_1,\dots,S_k$ 
. More precisely, let $g_1,\dots,g_k$ be a set of orthonormal indicator vectors for $A_1,\dots,A_k$ with Rayleigh quotients $\gamma_i \coloneq g_i^* M g_i$. Compute the value
$
    \frac{1}{k} \frac{\sum_{i=1}^k \gamma_i - k \lambda_1}{\lambda_{k+1} - \lambda_1}.
$
Theorem~\ref{thm:general} and Lemma~\ref{lem:kmeans} ensure that, if this value is small, then $A_1,\dots,A_k$ are close to $S_1,\dots,S_k$ (up to permutation of the indices). 

Besides recovering the original structure theorem, Theorem~\ref{thm:general} is not confined to the normalised Laplacian of undirected graphs. In Section~\ref{sec:digraphs} we will show how to apply Theorem~\ref{thm:general} to Hermitian representations of digraphs. In particular, we do not necessarily need each $g_i$ to be a classical indicator vector: in certain cases (as in digraph clustering) it might be beneficial to choose a different set of ``indicator vectors''. Moreover, in digraphs, $\lambda_1$ might not necessarily be zero: subtracting $\lambda_1$ from both numerator and denominator on the RHS of our bounds can improve them substantially.

\begin{figure}[ht]
    \begin{subfigure}{0.43\textwidth}
    \setlength{\abovecaptionskip}{8pt}   
    \setlength{\belowcaptionskip}{0pt}
        \includegraphics[width=\textwidth]{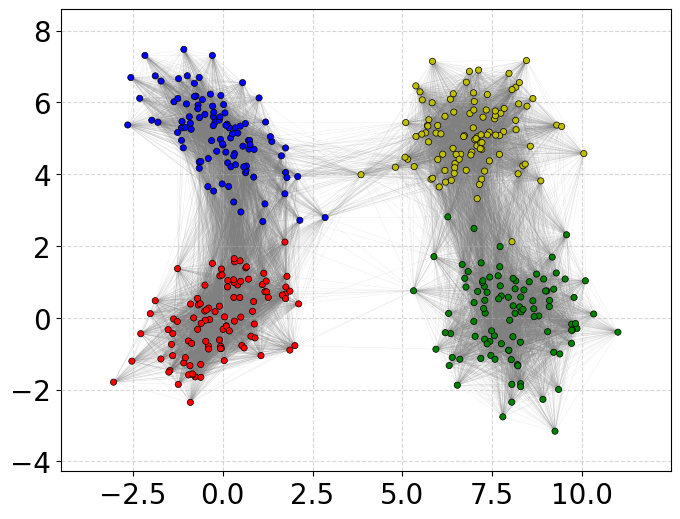}
        \caption{Geometric graph}
        \label{fig:4gaussianclustersgraph}
    \end{subfigure}
    \hfill
    \begin{subfigure}{0.48\textwidth}
    \setlength{\abovecaptionskip}{8pt}   
    \setlength{\belowcaptionskip}{0pt}
        \includegraphics[width=\textwidth]{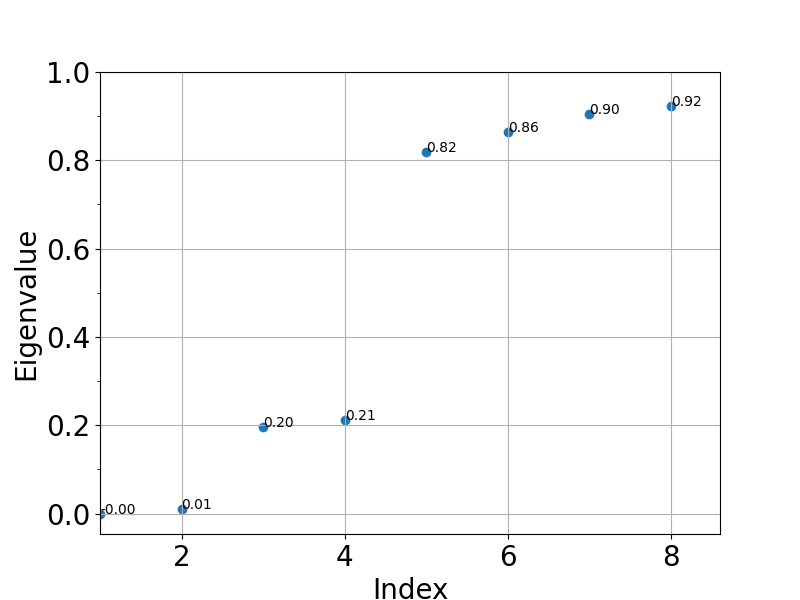}
        \caption{Smallest eigenvalues of the normalised Laplacian}
        \label{fig:4gaussianclusterseigenvalues}
    \end{subfigure}

    \caption{Four clusters generated by sampling points from a mixture of 4 Gaussians and the corresponding geometric graph (\ref{fig:4gaussianclustersgraph}). Notice how the smallest four eigenvalues come in pairs (\ref{fig:4gaussianclusterseigenvalues}).
    \label{fig:4GaussianClustersFigure}}
\end{figure} 
 
The main limitation of Theorem~\ref{thm:general} and Corollary~\ref{cor:structure} is that they both rely on the gap \(\gamma_i - \lambda_1\) being significantly smaller than \(\lambda_{k+1} - \lambda_1\). This, in practice, is not always satisfied. For example, when using \(M = \mathcal{L}\), we often see multiple gaps in the eigenvalues (which might result in some $\gamma_i$ being considerably larger than $\lambda_1=0$), with Spectral Clustering still working effectively (see, e.g., Figure~\ref{fig:4GaussianClustersFigure}).
To overcome this obstacle, we can recursively apply the following result to 
show that, as long as these gaps in the spectrum are large enough, Spectral Clustering can still be effective.

\begin{thm}[Recursive Structure Theorem] \label{thm:rec}
Let  \(q < k\) and $\hat{g}_1, \hdots, \hat{g}_q \in \spn\{{g}_1, \hdots, {g}_q\}$. 
Then, there exist 
\(\hat{g}_{q+1}, \hdots, \hat{g}_k \in \spn\{{g}_1, \hdots, {g}_k\}\)
and \(\hat{f}_{q+1}, \hdots, \hat{f}_k \in \spn\{{f}_{1}, \hdots, {f}_k\}\) such that \vspace{-0.2cm}
\[
\sum_{i=q+1}^{k} \|f_i - \hat{g}_i\|^2
 \leq \frac{\sum \limits_{i = q + 1}^{k} \left( \gamma_i  -  \lambda_{q + 1} \right)  + \lambda_{k + 1}  \sum \limits_{i=1}^{q} \|f_i - \hat{g}_i\|^2}{\lambda_{k + 1} - \lambda_{q + 1}} 
\]
The same bound also applies to \(\sum_{i=q+1}^{k} \|\hat{f}_i - g_i\|^2\)
\end{thm}

The idea behind Theorem~\ref{thm:rec} is that, if the indicator vectors can be expressed \emph{mostly} by the eigenvectors \(f_{q+1}, \hdots, f_k\), then by the orthonormality of both sets of vectors, the indicator vectors of \(g_1, \hdots, g_q\) can be expressed \emph{mostly} by \(f_1, \hdots, f_q\). This results in the error term \(\lambda_{k+1} \sum_{i=1}^q \|f_i - \hat{g}_i\|^2\) being small and in exchange we have an improved ratio to that of Theorem 1 provided \(\gamma_i < \lambda_{k+1}\) for \(i = q+1, \hdots, k\).

Theorem \ref{thm:rec} is most appropriate when there are multiple gaps in the eigenvalues of $M$: we split the spectrum into groups of eigenvalues of similar magnitude, with a large gap between each group; we then recursively reapply Theorem~\ref{thm:rec} for each group $\lambda_{q+1},\dots, \lambda_k$ as long as  \(\gamma_i - \lambda_{q+1}  \ll \lambda_{k+1} - \lambda_{q+1}\).

\begin{sproof}
The optimal choice of \(\hat{g}_i\) for \(i = q, \hdots, k\) is the projection of the indicator vectors onto the first \(k\) eigenvectors. Extending the normalised indicator vectors to an orthonormal basis forming orthogonal matrix \(G\), we obtain a orthogonal projection matrix \(Q\) where \(F = GQ^{*}\). The quantity \(\sum_{i=1}^{k}\|f_i - \hat{g}_i\|^2\) can be rewritten using submatrices of \(F, G\) and \(Q\) as shown in the Figure~\ref{fig:matrix_approx}.
\newpage
\begin{figure}[h!]
    \centering
    \includesvg[width=\linewidth]{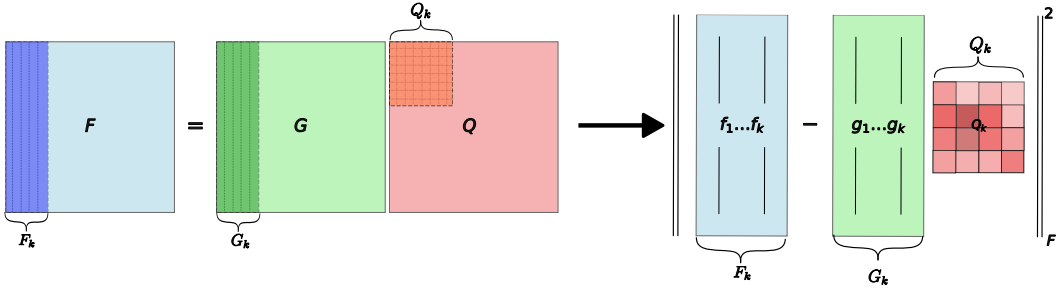}
    \caption{Illustration of how \(\sum_{i=1}^{k}\|f_i - \hat{g}_i\|^2\) is formed from orthogonal matrices.}
    \label{fig:matrix_approx}
\end{figure}
We rewrite \(\sum_{i=1}^{k}\|f_i - \hat{g}_i\|^2\) as \(\|F_k - G_kQ_k\|^2_F\) where \(F_k\) and \(G_k\) are the first \(k\) eigenvectors and indicator vectors respectively and \(Q_k\) is the \(k \times k\) top left block of the diagonal of \(Q\). The proof of this theorem utilises this fact and breaks it down further, considering the case when \(Q_k\) has blocks of concentration, see Figure~\ref{fig:concentration_blocks}:
\begin{figure}[h!]
    \centering
    \includesvg[width=0.2\linewidth]{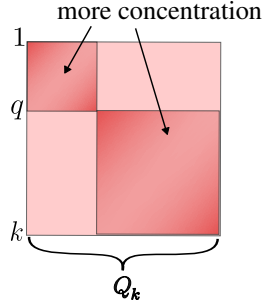}
    \caption{Illustration of how \(Q_k\) might have blocks with higher values.}
    \label{fig:concentration_blocks}
\end{figure}
\newline
The blocks correspond to groups of the indicator vectors that can express, with high accuracy, groups of eigenvectors.
The rows/columns of \(Q\) have unit length so \(\sum_{i=q+1}^{k} \|f_i - \hat{g}_i\|^2\) can be bounded above by \((k-q) - \sum_{i=q+1}^k \sum_{j=q+1}^k |Q_{ij}|^2\). The sum in this expression can then be bounded below as follows where we have again exploited the unit length rows of \(Q\):

\[
\gamma_i = g_i^*Mg_i = \sum_{j=1}^N \lambda_{j} |Q_{ij}|^2 \geq (\lambda_{q + 1} - \lambda_{k + 1}) \sum_{j = q + 1}^{k} |Q_{ij}|^2 + \lambda_{k+1} \left(1 - \sum_{j = 1}^q |Q_{ij}|^2 \right)
\]

Rearranging for \(\sum_{j = q + 1}^{k} |Q_{ij}|^2\)  and summing over \(i = q+1, \hdots, k\) provides us with with a lower bound on \(\sum_{i=q+1}^k \sum_{j=q+1}^k |Q_{ij}|^2\). Consequently, this gives an upper bound on \((k-q) - \sum_{i=q+1}^k \sum_{j=q+1}^k |Q_{ij}|^2\) as required. 
\end{sproof}

Typically, the choice for the indicator vectors \(g_1, \hdots, g_k\) is \(g_i =\frac{D^{1/2} \chi_i}{\|D^{1/2} \chi_i\|}\) where $\chi_i(u) = 1$ if $u \in S_i$ and $\chi_i(u) = 0$ otherwise.
The only real restrictions on our choice of indicator vectors, however, are that they are orthonormal and that \(\{D^{-1/2}g_i\}_{i=1}^k\) are constant on the clusters. Since the first eigenvector \(f_1\) of \(\mathcal{L}\) (or \(L\)) is the uninformative vector \(f_1 =\frac{D^{1/2} 1}{\|D^{1/2} 1\|}\) (or \(f_1 = \frac{1}{\|1\|}\)), we can choose \(g_1 = f_1\) and then re-orthogonalise \(g_2, \hdots, g_k\). This allows us to prove the following corollary.

\begin{corollary}\label{cor:RemoveFirstEvec}
    Let $\mathcal{G}$ be an undirected graph and $M = \mathcal{L}$. Let $g_1,\dots,g_k \in \mathbb{R}^N$ be orthonormal such that $g_1 = f_1$. Let \(\gamma_i = g_i^*Mg_i\)  \((i = 1, \hdots, k)\). Suppose $\lambda_{k+1} > \lambda_2$. Then, for \(i=1, \hdots, k\), there exists \(\hat{g}_i \in \spn\{g_1, \hdots, g_k\}\) such that \vspace{-0.2cm}
    \[
    \sum_{i=1}^k \|f_i - \hat{g}_i\|^2 \leq \frac{\sum_{i=2}^k (\gamma_i - \lambda_2)}{\lambda_{k+1} - \lambda_2}.
    \]
\end{corollary}

Corollary~\ref{cor:RemoveFirstEvec} essentially states that if the indicator vectors have Rayleigh quotient close to $\lambda_2 \ll \lambda_{k+1}$, then the bottom $k$ eigenvectors of $M$ must be very close to linear combinations of the indicator vectors. In contrast with the original structure theorem, this does not require $\lambda_2$ to be small. 

Corollary~\ref{cor:RemoveFirstEvec} can be used to analyse Spectral Clustering on stochastic block models. While such an analysis can be obtained using standard perturbation arguments, stochastic block models are an egregious instance where the structure theorem of Corollary~\ref{cor:structure} fails. We illustrate the performance of Corollary~\ref{cor:RemoveFirstEvec} on SBMs in the appendix.

\subsection{Experimental results}

To illustrate the impact of Theorem~\ref{thm:rec}, we consider synthetic graphs with a \emph{hierarchical} structure and show the improved performance the theorem provides. We then consider some real-world networks to show that this hierarchical structure is naturally occurring and can be exploited to get a better bound on the performance of Spectral Clustering. 

\paragraph{Geometric random graphs}
We sample \(100\) points each from four two-dimensional Gaussians centred, respectively, \((0,0),(0,5),(d,0),(d,5)\) where \(d\) is a parameter we vary. 
A graph is constructed by assigning an edge between any two points if the Euclidean distance between them is less than a threshold, which we choose in this case to be 4. In Figure~\ref{fig:combined_bounds_models}a we compare our bounds on the distance between indicator vectors of the clusters and the eigenvectors of the Laplacian with the results of \citet{macgregor2022tighter} and the actual distances. Results are averaged over 10 realisations.

\begin{figure}[ht]
    \centering
    \begin{minipage}[t]{0.45\textwidth}
        \centering
        \includegraphics[width=\linewidth]{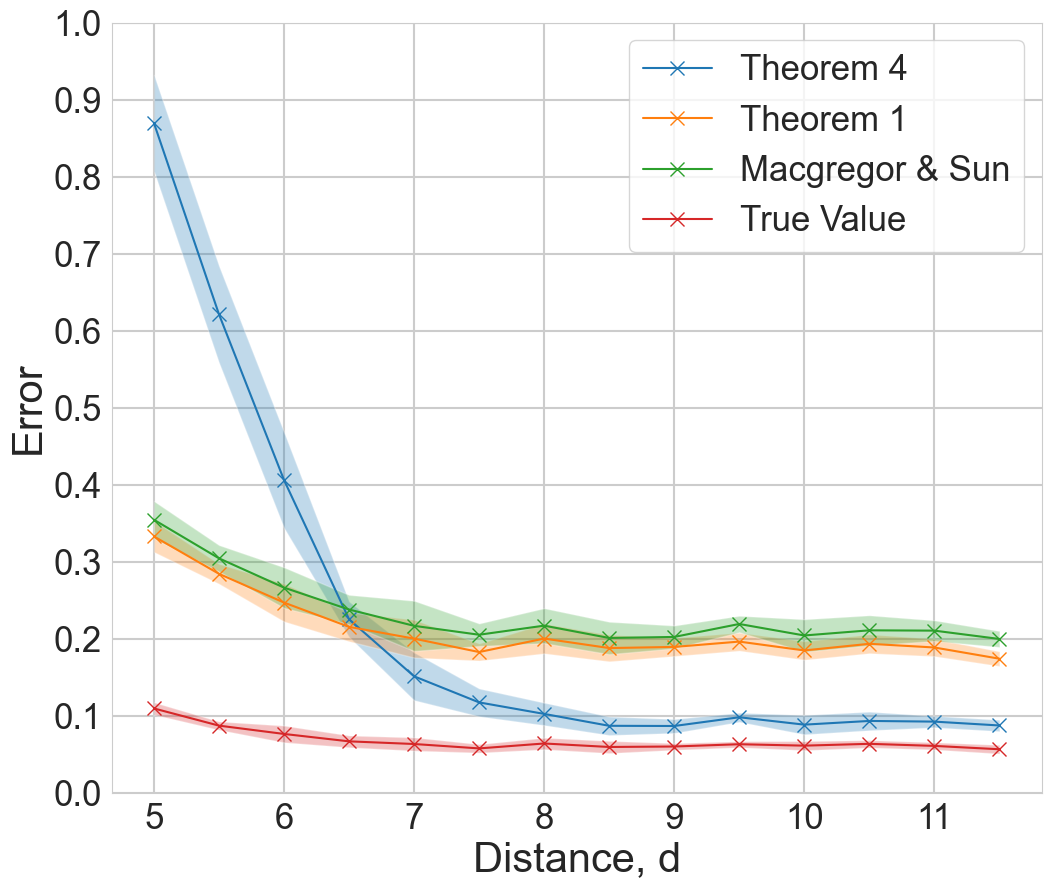}
        \caption*{(a)}
    \end{minipage}
    \hfill
    \begin{minipage}[t]{0.45\textwidth}
        \centering
        \includegraphics[width=\linewidth]{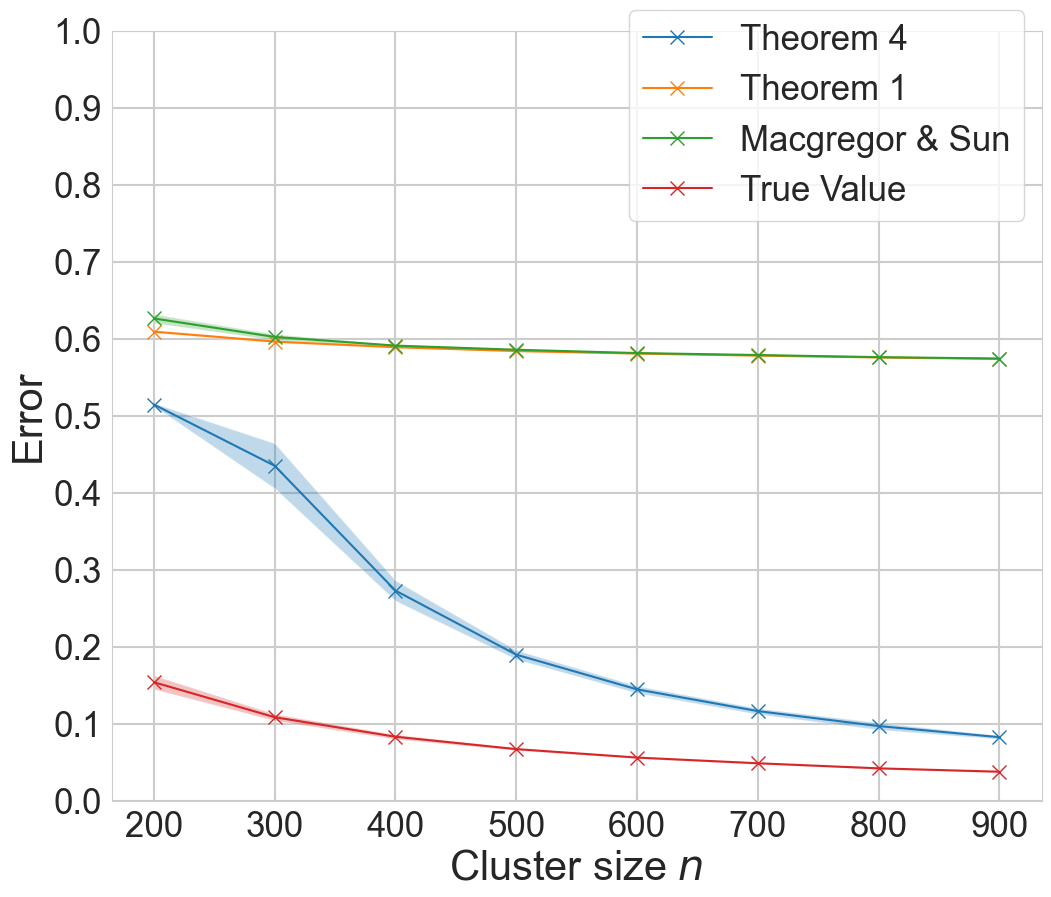}
        \caption*{(b)}
    \end{minipage}

    \caption{(a) Geometric random graph from Gaussian mixture model with varying distance between centres. (b) Stochastic block model with 4 blocks where two pairs have a higher affinity to each other. In both cases, the Error refers the bound given by Theorem \ref{thm:general} and \ref{thm:rec}, and by \citet{macgregor2022tighter} on $\frac{1}{k} \sum_{i=1}^{k} \|f_i - \hat{g}_i\|^2$, together to its actual value. Standard deviation is included as filled error bars.}
    \label{fig:combined_bounds_models}
\end{figure}

Notice that, as the distance parameter \(d\) increases, drawing two pairs of clusters further apart from each other, Theorem~\ref{thm:rec} drastically outperforms Corollary~\ref{cor:structure} and Theorem~\ref{thm:general}.

\paragraph{Stochastic block models}
We consider an SBM with 4 equal-sized clusters $S_1,\dots,S_4$. Let $P_{ij}$ be the probability that, for any $u \in S_i, v \in S_j$ there exists an edge between $u$ and $v$. We set $P_{ii} = 0.5 \,(i=1,\dots,4)$ and $P_{12} = P_{21} = P_{34} = P_{43} = 0.4$. All other values are set equal to \(0.1\). We effectively divide the vertices into two pairs of clusters which are more strongly connected to one another. In Figure~\ref{fig:combined_bounds_models}b, we compare our results with \cite{macgregor2022tighter} and the true distances between indicator vectors of the optimal clusters and the eigenvectors of the Laplacian. Notice that Theorem~\ref{thm:rec} greatly outperforms Corollary~\ref{cor:structure} and Theorem~\ref{thm:general}. Further experiments (such as when the parameters of the stochastic block model are close to the detectability threshold) are available in the Appendix.

\paragraph{Real-world networks}
In Table~\ref{tab:realworldgraphs}, we illustrate the improvement Theorem~\ref{thm:rec} provides over the results of \citet{macgregor2022tighter} on a variety of real-world networks. For MNIST \citep{MNIST}, Fashion MNIST \citep{xiao2017fashion} 
and the Air Quality dataset \cite{AirQuality}, we constructed a graph from the data by computing a correlation matrix from its data points and assigning edges based on whether the correlation exceeded a pre-defined threshold. Twitch \citep{rozemberczki2021twitch}, LastFM \citep{LastFMfeather}, Athletes \citep{facebook-gemsec} and CA-CondMat \citep{CondMatterCollab} are all network datasets available at SNAP \cite{snapnets}. $N$ (resp. $M$) refers to the number of nodes (resp. edges). The number of clusters $k$ has been chosen so that a relatively large gap between $\lambda_{k}$ and $\lambda_{k+1}$ exists. We include  The improvements given by Theorem~\ref{thm:general} over \cite{macgregor2022tighter} are due to the fact that we sum over the Rayleigh quotients of the $k$ indicator vectors, rather than simply upper-bound this value by $k \rho(k)$ as done in \cite{macgregor2022tighter}. More interesting is, arguably, the improvement of Theorem~\ref{thm:rec}, which is due to many networks having clustered eigenvalues and/or large $\lambda_2$. Notice how Theorem~\ref{thm:rec} is able to certify that, in these examples, the output of spectral clustering is well-correlated with the partition of minimum conductance, something not achievable with previous structure theorems. For added context, we include \(\lambda_{k+1}\) and an approximation of \(\rho(k)\) which we denote as \(\Tilde{\rho}(k)\). The approximation is obtained using the clusters outputted by spectral clustering.
\begin{table}[ht]
    \centering
    \caption{Comparison on the bounds on $\frac{1}{k} \sum_{i=1}^{k} \|f_i - \hat{g}_i\|^2$ given by Theorem~\ref{thm:general} and \ref{thm:rec}, and by~\citet{macgregor2022tighter} on real-world networks. We take a sample for MNIST and Fashion MNIST. }\label{tab:realworldgraphs}
    \begin{tabular}{lrrrccccc}
        \toprule 
         \bfseries Dataset & \bfseries k & \bfseries N & \bfseries M & \bfseries ~\cite{macgregor2022tighter} & \bfseries Theorem~\ref{thm:general} & \bfseries Theorem~\ref{thm:rec} & $\Tilde{\rho}(k)$ & $\lambda_{k+1}$ \\
        \midrule 
        MNIST*        & 5   & 2348    & 140018   & 0.4395  & 0.1915  & 0.1744  & 0.021 & 0.048\\
        Fash. MNIST*  & 6   & 2872    & 1711206  & 0.5614  & 0.3761  & 0.3225  & 0.200 & 0.35\\
        Air Quality           & 3   & 4942    & 2784780  & 0.6234  & 0.2883  & 0.2350  & 0.332 & 0.533\\
        Twitch                & 2   & 168114  & 6797557  & 0.9567  & 0.4793  & 0.3937  & 0.128 & 0.133\\
        LastFM                & 2   & 7622    & 27800    & 0.8373  & 0.4221  & 0.3123  & 0.013 & 0.016\\
        Athletes              & 3   & 13866   & 86852    & 0.6951  & 0.4033  & 0.2427  & 0.053 & 0.037\\
        CA-CondMat            & 3   & 21363   & 91342    & 0.7806  & 0.4818  & 0.3370  & 0.012 & 0.016\\
        \bottomrule 
    \end{tabular}
\end{table}


\section{Digraph clustering} \label{sec:digraphs}
We now apply our techniques to Hermitian representations of digraphs, where the adjacency matrix is specified as in Equation~\ref{eq:adjdig}. \citet{cucuringu2020hermitian} have shown experimentally that Spectral Clustering on these matrix representations is able to recover clusters characterised by large imbalances in the direction of inter-cluster edges, in the sense that most edges between two clusters $S_i$ and $S_j$ follow the same direction. This is a \emph{higher-order} clustering problem, since clusters are defined according to their inter-cluster relations~\cite{martin_comnet}. This is in contrast with traditional undirected graph clustering, in which  a cluster is usually defined only according to its inner and outer density.

\citet{cucuringu2020hermitian} have proposed an analysis of Spectral Clustering on a directed analogue of the classical stochastic block model, while \citet{laenen2020higher} have attempted to provide an analysis for more general graphs. In Section~\ref{sec:laenen}, however, we will argue that Laenen and Sun's results fail to explain the practical performances of Spectral Clustering on digraphs.
Here we attempt to remedy this gap in the literature. We begin by defining a cost function for the task.

\begin{definition}
Let $\mathcal{G}=(V,E,w)$ be a digraph and let $k \ge 2$. Let $S_1,\dots,S_k$ be a $k$-way partition of $V$. We define the \emph{cyclic expansion} of $S_1,\dots,S_k$ as
\[
\Psi(S_1,\dots,S_k) = \frac{1}{\vol(V)} \sum_{(i,j) \notin C_k} w(S_i,S_j),
\]
where \(C_k = \{ (i,j) \ | \ j \equiv i+1 \mod k, \ 1 \leq i,j \leq k \}\), and the \emph{cyclic $k$-way expansion} of $\mathcal{G}$ as
\[
\Psi_k(\mathcal{G}) = \min_{\Part\{S_i\}_{i=1}^k} \Psi(S_1,\dots,S_k).
\]
\end{definition}

We want to find clusters $S_1,\dots,S_k$ so that most of the out-going edges from $S_i$ are connected to vertices in $S_{i+1 \text{ mod } k}$. 
When $k=2$, our problem simply becomes finding \emph{disassortative} clusters. Indeed, the Hermitian Laplacian for $k=2$ is just the signless Laplacian, whose bottom eigenvectors are known to contain information about disassortative clusters \citep{maxcut,shipingliu}.
The next lemma clarifies the connection between this cost function and Hermitian Laplacians for digraphs.
\begin{lem} \label{lem:EquivalenceLambda}
    Let $\mathcal{G}$ be a connected digraph. Then, $\lambda_1(\mathcal{L})=0$ if and only if $\Psi_k(\mathcal{G}) = 0$.
\end{lem}

As observed in \cite{lisunzanetti}, the bottom eigenvector of a Hermitian Laplacian already contains information about all $k$ clusters. For this reason, given a $k$-way partition $\mathcal{S} = \{S_1,\dots,S_k\}$, we define $\chi_{\mathcal{S}} \in \mathbb{C}^N$ as follows. For any $j =1,\dots,k$, 
\[
\chi_{\mathcal{S}}(u) = \sqrt{\frac{d(u)}{\vol(V)}} \cdot \mathrm{e}^{\frac{2\pi \mathrm{i} j}{k}} \text{ for } u \in S_j,
\]
where $\mathrm{i}$ is the imaginary unit. Fundamentally, $\chi_{\mathcal{S}}$ maps each cluster to a power of the $k$-th root of unity.
The following simple but crucial lemma relates the Rayleigh quotient of $\chi_{\mathcal{S}}$ to $\Psi(S_1,\dots,S_k)$.

\begin{lem} \label{lem:PsiBounds}
Let $k \ge 2$. It holds that 
$16k^{-2} \Psi(S_1,\dots,S_k) \leq \chi_{\mathcal{S}}^* \mathcal{L}\chi_{\mathcal{S}} \le 4 \Psi(S_1,\dots,S_k).$
\end{lem}

We are now ready to state our structure theorem for digraphs.

\begin{thm}[Structure Theorem for Digraphs] \label{thm:digraph}
Let $\mathcal{G}$ be a digraph with Hermitian Laplacian $\mathcal{L}$. Assume $\lambda_2 > \lambda_1$. Given any $k$-way partition $\mathcal{S} = \{S_1,\dots,S_k\}$, there exists \(\alpha \in \mathbb{C}\) such that 
\begin{equation}
\label{eq:ours_ray}
\| f_1 - \alpha\chi_{\mathcal{S}}\|^2 \leq  \frac{\chi_{\mathcal{S}}^* \mathcal{L}\chi_{\mathcal{S}} - \lambda_1}{\lambda_2 - \lambda_1}.
\end{equation}
Furthermore, if $\mathcal{S}$ achieves $\Psi_k(\mathcal{G})$, then
\vspace{-0.1cm}
\begin{equation}
\label{eq:ours_psi}
\| f_1 - \alpha\chi_{\mathcal{S}}\|^2 \leq  \frac{4 \Psi_k(\mathcal{G}) - \lambda_1}{\lambda_2 - \lambda_1}.
\end{equation}
\end{thm} \vspace{-0.2cm}

We observe that 
$\displaystyle
\chi_{\mathcal{S}}^* \mathcal{L}\chi_{\mathcal{S}} = \frac{4}{\text{vol}(V)} \sum_{i=1}^k \sum_{j=1}^k w(S_i,S_j) \sin^2 \left( \frac{(i-(j+1)) \pi}{k} \right).$
Therefore, inequality (\ref{eq:ours_ray}) is generally stronger than (\ref{eq:ours_psi}) because it assigns a smaller penalty to edges from $S_i$ to $S_j$ when $i - (j+1)$ is small.

\subsection{Comparison with previous work}
\label{sec:laenen}
We now compare our structure theorem for digraphs to the one of Laenen and Sun \cite{laenen2020higher}. They consider the following  alternative to our cyclic expansion:
\[
\theta_k(\mathcal{G}) \triangleq \max_{\{S_i\}_{i=1}^k \text{ partition}}\sum_{i=1}^{k-1} \frac{w(S_i,S_{i+1})}{\text{vol}(S_i) + \text{vol}(S_{i+1})}.
\]
There are two differences compared with our definition of cyclic expansion. First, $\theta_k$ penalises edges from $S_k$ to $S_1$, i.e. it tries to fit a path rather than a directed cycle between clusters. While reasonable, we argue  this is not what Spectral Clustering on Hermitian Laplacians is actually doing. Secondly, and more importantly, the two cost functions differ in the normalisation based on the volume of the graph. We believe the normalisation chosen by \cite{laenen2020higher} is poorly suited to analyse Spectral Clustering. 

To motivate our  assertions, let us delve deeper into their results. We first remark  they choose a Hermitian Laplacian constructed with the $\lceil 2 \pi k \rceil$-th root of unity. They also construct a different ``indicator'' vector $\tilde{\chi}_{\mathcal{S}}$ for a $k$-way partition $\mathcal{S}$. This choice is not overly important, so we refer to their paper for further details. Their main result is as follows.

\begin{thm}[\cite{laenen2020higher}] \label{thm:laenenDigraphST}
    Let $f_1$ be the bottom eigenvector of the Hermitian Laplacian constructed with the \(\lceil 2 \pi k\rceil\)-th root of unity. Let $\eta_k(\mathcal{G}) \triangleq \frac{\lambda_2}{1 - (4/k)\theta_k(\mathcal{G})}$ and assume $\eta_k(\mathcal{G}) > 1$. Let $\mathcal{S}$ be a partition maximising $\theta_k(\mathcal{G)}$.
    There exists \(\beta \in \mathbb{C}\) such that 
$
\|f_1 - \beta \tilde{\chi}_{\mathcal{S}} \|^2 \leq (\eta_k(\mathcal{G}) - 1)^{-1}.
$
\end{thm}

We now compare Theorem~\ref{thm:digraph} with Theorem~\ref{thm:laenenDigraphST} on the two very simple digraphs of Figure~\ref{fig:cyclepath},

\begin{figure}[h!]
    \centering

    \begin{minipage}{0.36\textwidth}
        \centering
        \vspace{0.1cm}
        \begin{adjustbox}{valign=t}  
            
            \begin{subfigure}[b]{0.48\textwidth}
                \hspace*{-0.5cm}{
                \begin{tikzpicture}[scale=0.3,
                    node/.style={circle, draw, fill=blue!30, minimum size=0.1cm, inner sep=0pt},
                    edge/.style={->, thin, line width=0.4pt, opacity=0.5}]
                    \definecolor{cluster1color}{RGB}{240,128,128}
                    \definecolor{cluster2color}{RGB}{144,238,144}
                    \definecolor{cluster3color}{RGB}{135,206,250}
                    \definecolor{cluster4color}{RGB}{221,160,221}
                    \definecolor{cluster5color}{RGB}{255,182,193}
                    \foreach \i in {1,...,5} { \coordinate (C\i) at ({72*(\i-1)}:4cm); }
                    \foreach \i/\color in {1/cluster1color,2/cluster2color,3/cluster3color,4/cluster4color,5/cluster5color} {
                        \node[minimum size=2cm] (cluster\i) at (C\i) {};
                        \foreach \j in {1,...,5} {
                            \node[node, fill=\color] (v\i\j) at ($(cluster\i) + ({360/5*(\j-1)}:0.8cm)$) {};
                        }
                    }
                    \foreach \i in {1,...,4} {
                        \foreach \j in {1,...,5} {
                            \foreach \k in {1,...,5} {
                                \draw[edge] (v\i\j) -- (v\the\numexpr\i+1\relax\k);
                            }
                        }
                    }
                    \foreach \j in {1,...,5} {
                        \foreach \k in {1,...,5} {
                            \draw[edge] (v5\j) -- (v1\k);
                        }
                    }
                \end{tikzpicture}}
                \caption{Cycle.}
                \label{fig:Perfect5Cycle}
            \end{subfigure}%
            \hfill
            \begin{subfigure}[b]{0.48\textwidth}
                \centering
                \begin{tikzpicture}[scale=0.3,
                    node/.style={circle, draw, fill=blue!30, minimum size=0.1cm, inner sep=0pt},
                    edge/.style={->, thin, line width=0.4pt, opacity=0.5}]
                    \definecolor{cluster1color}{RGB}{240,128,128}
                    \definecolor{cluster2color}{RGB}{144,238,144}
                    \definecolor{cluster3color}{RGB}{135,206,250}
                    \definecolor{cluster4color}{RGB}{221,160,221}
                    \definecolor{cluster5color}{RGB}{255,182,193}
                    \coordinate (C1) at (0,0);  
                    \coordinate (C2) at (5,-2); 
                    \coordinate (C3) at (0,-4); 
                    \coordinate (C4) at (5,-6); 
                    \coordinate (C5) at (0,-8);
                    \foreach \i/\color in {1/cluster1color,2/cluster2color,3/cluster3color,4/cluster4color,5/cluster5color} {
                        \node[minimum size=2cm] (cluster\i) at (C\i) {};
                        \foreach \j in {1,...,5} {
                            \node[node, fill=\color] (v\i\j) at ($(cluster\i) + ({360/5*(\j-1)}:0.8cm)$) {};
                        }
                    }
                    \foreach \i in {1,...,4} {
                        \foreach \j in {1,...,5} {
                            \foreach \k in {1,...,5} {
                                \draw[edge] (v\i\j) -- (v\the\numexpr\i+1\relax\k);
                            }
                        }
                    }
                \end{tikzpicture}
                \caption{Path.}
                \label{fig:DirectedPath5}
            \end{subfigure}
            \end{adjustbox}
            \caption{Examples of directed cluster structures.}
        \label{fig:cyclepath}
        
    \end{minipage}%
    \hspace{0.3cm}%
    \begin{minipage}{0.6\textwidth}
        \centering
            \includegraphics[width=0.8\textwidth]{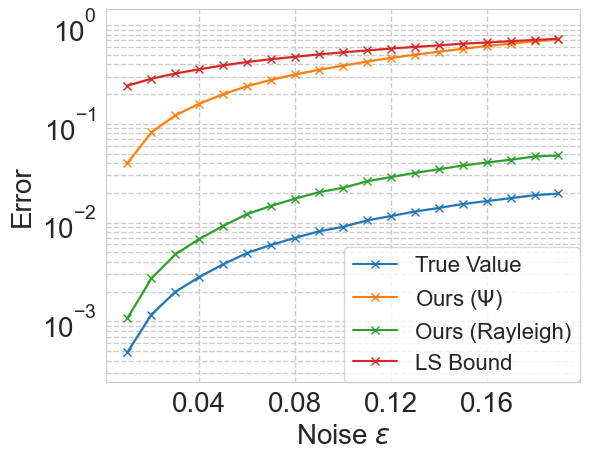}
        \caption{Comparison of the results given by Theorem~\ref{thm:digraph} (green for (\ref{eq:ours_ray}) and orange for (\ref{eq:ours_psi})) and by \cite{laenen2020higher} (red) for a cyclic DSBM at varying level of noise. The actual values are reported in blue. Averaged over 10 realisations.}
        \label{fig:5PathImage}
    \end{minipage}

\end{figure}

in which we have five clusters perfectly arranged on a directed cycle (Figure~\ref{fig:cyclepath}a)  and on a directed path (Figure~\ref{fig:cyclepath}b). Clearly, in both cases, $\Psi_k = 0$. Applying our structure theorem, this implies that, in both cases, the first eigenvector is exactly a multiple of the indicator vector $\chi_{\mathcal{S}}$. Therefore, the clusters are embedded in $k$ distinct and well-separated points, which are just the $k$ powers of the $k$-th root of unity rotated by some $\alpha \in \mathbb{C}$. Our structure theorem correctly predicts that spectral clustering will recover the clusters perfectly.

If we apply Theorem~\ref{thm:laenenDigraphST} instead, we obtain that, for the digraph of Figure~\ref{fig:Perfect5Cycle}, $\|f_1 - \beta \tilde{\chi}_{\mathcal{S}} \|^2 \le (\eta_k(\mathcal{G}) - 1)^{-1} \approx 0.642$ which is not informative at all. If we consider the digraph of Figure~\ref{fig:DirectedPath5}, we obtain a slightly better bound: $\|f_1 - \beta \tilde{\chi}_{\mathcal{S}} \|^2 \le (\eta_k(\mathcal{G}) - 1)^{-1} \approx 0.294$. This is still far from the true value, which is equal to zero.

\subsection{Experimental results}

\paragraph{Directed stochastic block models} 
We now apply our results to the Directed Stochastic Block Model of \cite{cucuringu2020hermitian}.
Given parameters $k\ge 2,n \ge 1, P \in [0,1]^{k \times k}, F \in [0,1]^{k \times k}$, a directed stochastic block model \( \mathcal{G} \sim \text{DSBM}(k, n, P, F) \) is a random graph of $N=kn$ vertices constructed as follows: each vertex belongs to one of $k$ communities $S_1,\dots,S_k$ of $n$ vertices each. We place an edge independently at random between any two vertices $u \in S_i,v \in S_j$ with probability $P_{ij} = P_{ji}$. Furthermore, we orient the edge between $u$ and $v$ from $u$ to $v$ with probability $F_{ij}$, from $v$ to $u$ with probability $F_{ji} = 1 - F_{ij}$.

We consider a model $\mathcal{G} \sim \text{DSBM}(k, n, P, F)$ for $k=4$ and $n=100$ with $P,F$  specified as follows. \vspace{-0.2cm}

\small
\[F = \begin{pmatrix}
    &\cellcolor{yellow!30}.5 & \cellcolor{yellow!60}1 & \cellcolor{yellow!30}.5 & \cellcolor{yellow!30}.5 &\\
    &\cellcolor{yellow!0}0 & \cellcolor{yellow!30}.5 & \cellcolor{yellow!60}1 & \cellcolor{yellow!30}.5 &\\
    &\cellcolor{yellow!30}.5 & \cellcolor{yellow!0}0 & \cellcolor{yellow!30}.5 & \cellcolor{yellow!60}1  &\\
    &\cellcolor{yellow!30}.5 & \cellcolor{yellow!30}.5 & \cellcolor{yellow!0}0 & \cellcolor{yellow!30}.5 &\\
\end{pmatrix}, \
P = \begin{pmatrix}
    &\epsilon & \cellcolor{yellow!60}1 & \epsilon & \epsilon &\\
    &\cellcolor{yellow!60}1 & \epsilon & \cellcolor{yellow!60}1 & \epsilon &\\
    &\epsilon & \cellcolor{yellow!60}1 & \epsilon & \cellcolor{yellow!60}1 &\\
    &\epsilon & \epsilon & \cellcolor{yellow!60}1 & \epsilon & \\
\end{pmatrix}\]
\normalsize

$F$ represents a path structure, while this choice of $P$ means the path structure is very pronounced. \(\epsilon\) is a noise parameter: the smaller $\epsilon$ is, the closer the graph will be to having a perfect path cluster-structure.

In Figure~\ref{fig:5PathImage}, we compare the bounds of Theorem~\ref{thm:digraph} with the results of \cite{laenen2020higher}, and with the true distance between the bottom eigenvector of the Hermitian Laplacian and the indicator vector of the clusters.  Our results predict the true values exceptionally well and correctly imply Spectral Clustering will work almost perfectly for all noise levels considered. On the contrary, Laenen and Sun's result becomes non-informative even for small noise parameters. We provide a similar experiment for a DSBM with a cyclic cluster structure in the Appendix.

\begin{table}[h]
    \centering
    \captionof{table}{Comparison on the bounds on $\|f_1 -\alpha \chi_{\mathcal{S}}\|^2$ given by Theorem \ref{thm:digraph} (\ref{eq:ours_ray}) (Ours) and \citet{laenen2020higher} (LS) on real-world networks.}\label{tab:digraphs}
    \begin{tabular}{lcccccc}
        \toprule
        Network & $k$ & N & M & $\Psi$ & Ours & LS \\
        \midrule
        Yellowstone & 4 & 15 & 37 & 0.027 & 0.086 & 0.662 \\
        COVID-19    & 4 & 67 & 66 & 0.000 & 0.000 & N/D   \\
        St. Marks   & 5 & 49 & 226 & 0.154 & 0.324 & N/D   \\
        St. Martin  & 4 & 45 & 224 & 0.118 & 0.352 & N/D   \\
        Ythan       & 4 & 135 & 601 & 0.101 & 0.399 & N/D   \\
        \bottomrule
    \end{tabular}

\end{table}

\paragraph{Real-world directed networks} We apply our results to real-world directed networks and summarise our findings in Table~\ref{tab:digraphs}: for each network considered, we compare our bounds from Theorem~\ref{thm:digraph} with Laenen and Sun's Theorem~\ref{thm:laenenDigraphST}.

We first consider the directed graph analysed in \cite{laenen2020higher} made from the Data Science for COVID-19 Dataset \citep{kcdc2020}, where an edge from \(u\) to  \(v\) exists if \(u\) has infected \(v\). This graph has many disconnected components so we consider its largest weakly connected component. Spectral Clustering finds clusters fitting perfectly to a directed $4$-cycle; therefore, \(\Psi_4 = 0\) and  our bound in Theorem~\ref{thm:digraph} is $\|f_1 - \beta \chi_{\mathcal{S}} \|^2 = 0$. On the other hand, since $\eta_4 < 1$, Laenen and Sun's result is not applicable. Both this and the previous data set are characterised by unbalanced clusters, which results in a small $\theta_4$ and makes Laenen and Sun's results uninformative. 


The Yellowstone \citep{yellowstone_foodweb} data set a small digraph representing a food web for Yellowstone National Park \citep{yellowstone_foodweb}, where vertices represent animal species and there is an edge from $u$ to $v$ if $u$ is predated by $v$.  As shown in Figure~\ref{fig:YellowstoneTrophicCascade}, this network can be partitioned into four clusters exhibiting an almost perfect directed path structure: only the outgoing edges for \emph{Mule deer} are not consistent with this structure. Indeed, $\Psi_4 = 0.027$ and Spectral Clustering using the Hermitian Laplacian perfectly recovers these clusters. Theorem~\ref{thm:digraph} suggests an error bound less than $0.086$, which is close to the actual value \(\|f_1 - \beta \chi_{\mathcal{S}} \|^2 = 0.039\). On the other hand, Laenen and Sun's result can only obtain an upper bound of $0.662$, which is not indicative of the actual performance of Spectral Clustering.

St. Marks Seagrass, St. Martin Island, and Ythan Estuary data sets \citep{cosin_foodwebs_dataset} are networks representing other food webs. 
While these data sets do not present a cluster structure as obvious as the COVID-19 or Yellowstone data sets, our results imply there exists a nontrivial correlation between the indicator vector of the clusters and the bottom eigenvector of the Hermitian Laplacian. Notice that in all these three cases, \(\eta < 1\), making Laenen and Sun's bounds uninformative.

\begin{figure}[h]
    \centering
    \includegraphics[width=0.7\textwidth]{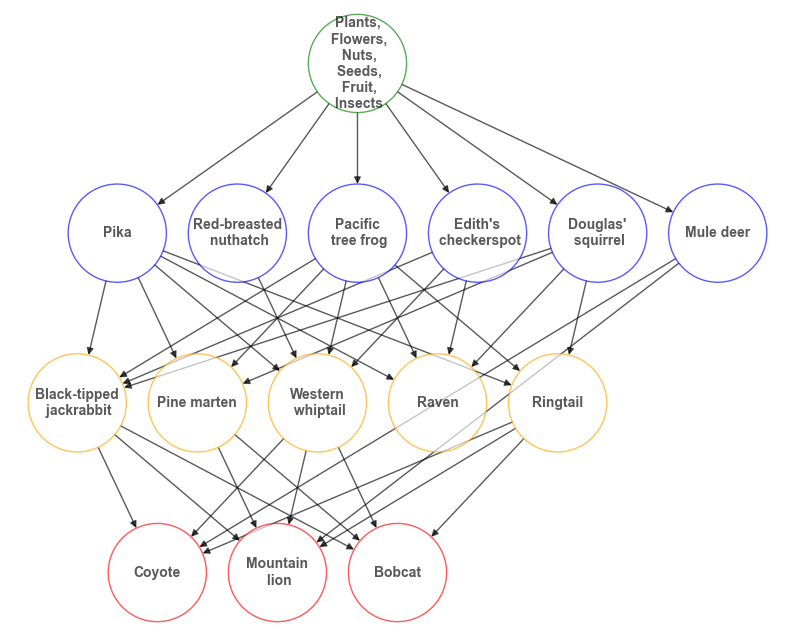}
    \caption{A directed graph illustrating a food web for Yellowstone National Park, United States.}
    \label{fig:YellowstoneTrophicCascade}
\end{figure}

\section{Conclusion}
We have presented a generalised framework for analysing the performance of Spectral Clustering. Our work generalises and improve upon the line of work based on the structure theorem of \citet{peng2015partitioning}. In particular, we have shown that spectral clustering works well as long as the bottom eigenvalues of the Laplacian matrix of an undirected graph can be divided in well-separated groups. Experimental results show that our bounds accurately predict the performance of spectral clustering on a wide range of both synthetic and real-world networks where previous analyses failed.

Furthermore, we have applied our techniques to analyse spectral clustering on Hermitian Laplacians for digraphs, as defined by \citet{cucuringu2020hermitian}. By exploiting a new cost function, which we call the cyclic $k$-way expansion, we are able to analyse the performances of spectral clustering on digraphs, vastly outperforming previous analyses.


\newpage

\bibliographystyle{unsrtnat}
\bibliography{reference}

\title{Proofs and Additional Experiments \\(Appendix)}
\maketitle
\appendix

 \section{Omitted proofs and results}
\subsection{Section~\ref*{sec:general}}
\subsubsection{Proof of Lemma~\ref*{lem:kmeans}}

\begin{proof}[Proof of Lemma~\ref{lem:kmeans}]
Let $c_i = \sum_{u \in A_i} d(u) \Tilde{F}(u,.)$ be the centroid of $A_i$ for $i=1,\dots,k$. Let $c \colon V \to \mathbb{R}^{k}$ be a map from a vertex to its corresponding centroid, i.e., $c(u) = c_i$ if $u \in A_i$.

Since $A_1,\dots,A_k$ is a $(1+\alpha)$ approximation of the optimal $k$-means partition, we have that

\begin{align}
\sum_{j = 1}^k \sum_{u \in A_j}d(u) \|\Tilde{F}(u,.) - c_j\|^2 
& = \sum_{j = 1}^k \sum_{u \in S_j} d(u)\|\Tilde{F}(u,.) -  c(u)\|^2 \nonumber \\
& \leq (1+\alpha) \sum_{j = 1}^k \sum_{u \in S_j}d(u) \|\Tilde{F}(u,.) - \mu_j\|^2 \nonumber \\
& \leq (1+\alpha) \sum_{j = 1}^k \sum_{u \in S_j} \|F(u,.) - G(u,.)\|^2 \leq (1+\alpha)U. \label{eq:ubound}
\end{align}

By applying the triangle inequality and the simple inequality $(x-y)^2 \ge \frac{x^2}{2} - y^2$, we obtain the following:
 \begin{align*}
\sum_{j = 1}^k \sum_{u \in S_j} d(u)\|\Tilde{F}(u,.) - c(u)\|^2 
& \geq \sum_{j = 1}^k \sum_{u \in S_j}d(u) (\|\mu_j - c(u)\| - \|\Tilde{F}(u,.) - \mu_j\|)^2 \\
& \geq \sum_{j = 1}^k \sum_{u \in S_j} \frac{1}{2}d(u)\|\mu_j - c(u)\|^2 - \sum_{j = 1}^k \sum_{u \in S_j} d(u)\|\Tilde{F}(u,.) -  \mu_j\|^2 \\
& \geq \sum_{j = 1}^k \sum_{u \in S_j} \frac{1}{2}d(u)\|\mu_j - c(u)\|^2 - U,
\end{align*}
where the last inequality follows from Equation~\ref{eq:ubound}. 


By the assumption \(\|\mu_i - \mu_j\| \geq \mathcal{D}\) for all \(i \ne j\), it holds that \(\|c_i - c_j\| \geq \mathcal{D}/2\). Otherwise, there would be \(\mu_\ell\) such that \(\|\mu_\ell - c_i\| \geq \mathcal{D}/2\) for any \(i\), but this would violate our assumption on $U$. Indeed, assume by contradiction \(\|\mu_l - c_i\| \geq \mathcal{D}/2\) for any \(i\). Then,
\begin{align*}
(1+\alpha) U &\ge \sum_{j = 1}^k \sum_{u \in S_j} d(u)\|\tilde{F}(u,.) - c(u)\|^2 \\ 
&\ge \frac{1}{2}\sum_{u \in S_\ell} d(u)\|\mu_\ell - c(u)\|^2 - U \\
&\ge \frac{1}{4}\cdot \mathcal{D} \cdot \vol(S_\ell) - U.
\end{align*}

Therefore, $U \ge \frac{\mathcal{D} \cdot \vol(S_\ell)}{4(2+\alpha)}$, contradicting our assumption on $U$.
\color{black}

We say that a vertex $u \in S_j$ is \emph{misclassified} if $\|c(u) - \mu_j\| \ge \mathcal{D}/2$. Let $\sigma \colon [k] \to [k]$ be the permutation of the indices minimising the volume of the symmetric difference $\sum_{i=1}^k \vol(S_i \setminus A_{\sigma(i)} \cup A_{\sigma(i)} \setminus S_i)$. Notice that we can obtain an upper bound to this volume by bounding the volume of misclassified vertices. We can upper bound the latter by noticing that, whenever we misclassify a vertex, we pay a price of $\Omega(D^2)$ in the $k$-means cost. Noting this, we obtain the following bound:
\begin{align*}
\sum_{j = 1}^k \sum_{u \in S_j} d(u)\|\Tilde{F}(u,.) - c(u)\|^2 
& \geq \sum_{j = 1}^k \sum_{u \in S_j} \frac{1}{2}d(u)\|\mu_j - c(u)\|^2 - U \\
& \geq \frac{\mathcal{D}^2}{8}  \text{vol}(\text{misclassified vertices}) - U,
\end{align*}

Together with Equation~\ref{eq:ubound}, this implies that
\[
\min_{\sigma} \sum_{i=1}^k \vol(S_i \setminus A_{\sigma(i)} \cup A_{\sigma(i)} \setminus S_i) \le \text{vol}(\text{misclassified vertices}) \leq \frac{8(2+ \alpha)U}{\mathcal{D}^2}.
\]
\end{proof}

In this section we prove the results from Section ~\ref*{sec:general}. 
For the following proofs, we must first define the matrices \(Q,R \in \mathbb{C}^{N \times N}\) which map between \(F\) and \(G\). Precisely,
\begin{align*}
    G = FQ \\
    F = GR.
\end{align*}
In the following lemma we prove some basic properties of \(Q\) and \(R\).
\begin{lem}\label{lem:PropertiesOfQ}
    We have the following basic properties of \(Q\):
\begin{enumerate}
    \item \(Q\) is orthogonal.
    \item \(R = Q^*\).
\end{enumerate}
\end{lem}
\begin{proof}
    As \(F\) is orthogonal, we have that \(Q = F^*G\). Noting that \(G\) is also orthogonal, we have that
    \[ Q^* Q = (F^*G)^*(F^*G) = G^*FF^*G = I\]
    and similarly, \(QQ^* = I\) so \(Q\) is orthogonal. \\
    For 2., Clearly,
    \[R = G^*F = Q^*\]
    as \(G\) is orthogonal.
\end{proof}

\subsubsection{Proof of Theorem~\ref*{thm:general} and Corollary ~\ref*{cor:structure}}
Below we use Lemma~\ref{lem:PropertiesOfQ} to prove our Theorem ~\ref{thm:general}.

\begin{proof}[Proof of Theorem~\ref*{thm:general}]

    As \(G = F Q\), this implies that \(g_i = \sum_{j=1}^N Q_{ij} f_j\). We choose \(\hat{f}_i = \sum_{j=1}^k Q_{ij} f_j\).  Now notice that,
    
    \[\gamma_i = \bar{g}_i^* M \bar{g}_i = \sum_{j=1}^N |Q_{ij}|^2 \lambda_j.\]
    As \(Q\) is orthogonal, \(\sum_{j=1}^N |Q_{ij}|^2 = 1\). This coupled with the assumption that \(\lambda_1 \leq \lambda_2 \leq \hdots \leq \lambda_N\) provides the following bound:
    
    \begin{align*}
        \gamma_i = \sum_{j=1}^N |Q_{ij}|^2 \lambda_j & \geq \lambda_1  \sum_{j=1}^k |Q_{ij}|^2 + \lambda_{k+1} (1- \sum_{j=1}^k |Q_{ij}|^2)
        \\  \Rightarrow \gamma_i & \geq \sum_{j=1}^k |Q_{ij}|^2 (\lambda_1 - \lambda_{k+1}) + \lambda_{k+1}
    \end{align*}
    Rearranging for \(\sum_{j=1}^k|Q_{ij}|^2\), we get the following:
    \[\sum_{j=1}^k |Q_{ij}|^2 \geq \frac{ \lambda_{k+1} - \gamma_i}{\lambda_{k+1}- \lambda_1}\]
    
    It follows that 
    \[1 - \sum_{j=1}^k |Q_{ij}|^2 \leq \frac{\gamma_i - \lambda_1}{\lambda_{k+1} - \lambda_1}\].
    As \(F = G R\), \(f_i = \sum_{j=1}^N R_{ij} g_j\). We choose \(\hat{g}_i = \sum_{j=1}^k R_{ij} g_j\).
    Now notice that,
    \[\|f_i - \hat{g}_i\|^2 = 1 - \sum_{j=1}^k |R_{ij}|^2.\]
    Summing over \(i=1, \hdots, k\) gives the following.
    \[\sum_{i=1}^k \|f_i - \hat{g}_i\|^2 = k - \sum_{i=1}^k \sum_{j=1}^k |R_{ij}|^2\]
    Recall from Lemma~\ref{lem:PropertiesOfQ} that \(Q^* = R\) so \( \sum_{i=1}^k \sum_{j=1}^k |R_{ij}|^2 = \sum_{i=1}^k \sum_{j=1}^k |Q_{ij}|^2\). So, it follows that:
    \[\sum_{i=1}^k \|f_i - \hat{g}_i\|^2 \leq \sum_{i=1}^k \frac{\gamma_i - \lambda_1}{\lambda_{k+1} - \lambda_1} .\]

\end{proof}

By a particular choice of \(\{g_i\}_{i=1}^k\) and \(M\) in Theorem~\ref{thm:general}, we obtain the structure theorem from \cite{macgregor2022tighter}.

\begin{proof}[Proof of Corollary~\ref*{cor:structure}]
     With the choice of \(g_i = \frac{D^{1/2}\chi_i}{\|D^{1/2}\chi_i\|}\) for \(i=1,\hdots,k\), we have the following:
    \begin{align*}
        \gamma_i & = g_i^* \mathcal{L} g_i \\
        & = \frac{\chi_i^*(D-A)\chi_i}{\chi_i^*D \chi_i} \\
        & = \frac{E(S_i, V - S_i)}{\text{vol}(S_i)} \\
        & \leq \max_{i=1, \hdots, k} \frac{E(S_i, V - S_i)}{\text{vol}(S_i)} = \rho(k)
    \end{align*}
Applying Theorem~\ref{thm:general} with \(M = \mathcal{L}\) and noting that \(\lambda_1 = 0\) completes the proof. 
\end{proof}

\subsubsection{Proof of Theorem~\ref*{thm:rec} and Corollary~\ref*{cor:RemoveFirstEvec}}
\begin{proof}[Proof of Theorem~\ref*{thm:rec}]
     Since \(f_1, \hdots, f_N\) form an orthonormal basis, we can write that \(g_i = \sum\limits_{j=1}^N Q_{ij}f_j\) and we choose \(\hat{f}_i = \sum\limits_{j=1}^q Q_{ij}f_j\) for \(i = 1, \hdots, q\) (summing only up to \(q\) by assumption) and \(\hat{f}_i = \sum\limits_{j=1}^k Q_{ij}f_j\) for \(i = q+1, \hdots, k\). Now let \(i\) be some index between \(q+1\) and \(k\),
    \begin{align*}
        \gamma_i = g_i^*Mg_i = \sum_{j=1}^N \lambda_{j} |Q_{ij}|^2 & \geq \lambda_1 \sum_{j=1}^{q} |Q_{ij}|^2 + \lambda_{q + 1} \sum_{j = q + 1}^{k} |Q_{ij}|^2 + \lambda_{k+1} \sum_{j = k+1}^n |Q_{ij}|^2 \\
        & = \lambda_1 \sum_{j=1}^{q} |Q_{ij}|^2  + \lambda_{q + 1} \sum_{j = q + 1}^{k} |Q_{ij}|^2 + \lambda_{k + 1}\left(1 -  \sum_{j = q + 1}^{k} |Q_{ij}|^2 -  \sum_{j=1}^{q} |Q_{ij}|^2 \right)
    \end{align*}
Rearranging for \(\sum\limits_{j=q + 1}^{k} |Q_{ij}|^2\), we get:

\begin{align*}
\sum\limits_{j=q + 1}^{k} |Q_{ij}|^2 & \geq \frac{\lambda_{k+ 1}\left( 1 - \sum\limits_{j=1}^{q} |Q_{ij}|^2 \right) - \gamma_i  + \lambda_1 \sum\limits_{j=1}^{q} |Q_{ij}|^2}{\lambda_{k + 1} - \lambda_{q + 1}} \\
& \geq \frac{\lambda_{k + 1}\left( 1 - \sum\limits_{j=1}^{q} |Q_{ij}|^2 \right) - \gamma_i}{\lambda_{k + 1} - \lambda_{q + 1}} \\
\end{align*}

Next, we sum over \(i=q+1, \hdots, k\) so we have that:
\begin{align*}
\sum\limits_{i= q + 1}^{k}\sum\limits_{j=q + 1}^{k} |Q_{ij}|^2
& \geq \frac{\lambda_{k + 1}\left( (k-q) - \sum\limits_{i= q + 1}^{k}\sum\limits_{j=1}^{q} |Q_{ij}|^2 \right) - \sum\limits_{i= q + 1}^{k}\gamma_i}{\lambda_{k + 1} - \lambda_{q + 1}}.
\end{align*}
Since \(\hat{f}_i \in \spn\{f_1, \hdots, f_q\}\) for \(i \in \{1, \hdots, q\}\), the sum in the brackets can be bounded as follows:
\[\sum\limits_{i= q + 1}^{k}\sum\limits_{j=1}^{q} |Q_{ij}|^2 \leq  \sum\limits_{i= q + 1}^{N}\sum\limits_{j=1}^{q} |Q_{ij}|^2 = q - \sum\limits_{i= 1}^{q}\sum\limits_{j=1}^{q} |Q_{ij}|^2 = \sum_{i=1}^q \|\hat{f}_i - g_i\|^2 \]
So,
\[
\sum\limits_{i= q + 1}^{k}\sum\limits_{j=q + 1}^{k} |Q_{ij}|^2
\geq \frac{\lambda_{k + 1}\left( (k-q) -  \sum\limits_{i=1}^{q} \|f_i - \hat{g}_i\|^2 \right) - \sum\limits_{i= q + 1}^{k}\gamma_i}{\lambda_{k + 1} - \lambda_{q + 1}}
\]
Now, noting that \(Q\) is orthogonal and therefore its rows are unit length, it follows that:
\[\sum\limits_{i= q + 1}^{k} \|\hat{f}_i - g_i\|^2 = 
(k-q) - \sum\limits_{i= q + 1}^{k}\sum\limits_{j= 1}^{k}|Q_{ij}|^2 \leq (k-q) - \sum\limits_{i= q + 1}^{k}\sum\limits_{j= q+1}^{k}|Q_{ij}|^2\]
Therefore,
\[\sum\limits_{i= q + 1}^{k} \|\hat{f}_i - g_i\|^2 \leq \frac{\sum \limits_{i=q+1}^k \gamma_i - (k-q)\lambda_{q+1} - \lambda_{k+1}\sum\limits_{i=1}^{q} \|\hat{f}_i - g_i\|^2 }{\lambda_{k + 1} - \lambda_{q + 1}}\]
Now, we show the same bound for \(\sum\limits_{i= q + 1}^{k} \|f_i - \hat{g}_i\|^2\). Choosing \(\hat{g}_i = \sum\limits_{j=1}^{q} R_{ij}f_j\) for \(i = 1, \hdots, q\), Lemma~\ref{lem:PropertiesOfQ} implies that
\[\sum\limits_{i=1}^{q} \|\hat{f}_i - g_i\|^2 = \sum\limits_{i=1}^{q} \|f_i - \hat{g}_i\|^2\]
as \(Q = R^*\).
Similarly, choosing \(\hat{g}_i = \sum\limits_{j=1}^{k} R_{ij}f_j\) for \(i = q+1, \hdots, k\) implies that 
\[\sum\limits_{i=q+1}^{k} \|\hat{f}_i - g_i\|^2 = \sum\limits_{i=q+1}^{k} \|f_i - \hat{g}_i\|^2\]
which completes the proof.
\end{proof}


\begin{proof}[Proof of Corollary~\ref*{cor:RemoveFirstEvec}]
    Choosing \(g_1 = f_1\) and remaining indicator vectors \(g_2, \hdots, g_k\) all orthonormal means the conditions for Theorem~\ref{thm:rec} are satisfied with \(q=1\). 
    \(\hat{g}_1 = g_1 = f_1\)  means that:
    \begin{enumerate}
        \item \(\|\hat{g}_1 - f_1\|^2 = 0\)
        \item \(\sum\limits_{i=2}^k \|f_i - \hat{g}_i\|^2 \leq \frac{\sum_{i=2}^k (\gamma_i - \lambda_2)}{\lambda_{k+1} - \lambda_2} \)
    \end{enumerate}
    
    which completes the proof.
\end{proof}

\subsubsection{Spectral Clustering performance on SBMs}

Using Corollary~\ref*{cor:RemoveFirstEvec} we have a result regarding the performance of Spectral Clustering on SBMs.

We generate $\mathcal{G} \sim SBM(n,k,p,q)$ as follows: we divide $kn$ vertices into $k$ communities $S_1,\dots,S_k$ of equal size $n$. For any $u \in S_i, v \in S_j$, we place an edge between $u$ and $v$ with probability $p$ if $i=j$, with probability $q$ otherwise. We recover the following well-known result.
\newline
\begin{corollary}
\label{cor:sbm}
Let \(\mathcal{G} \sim SBM(n,k,p,q)\) for some \(k \geq 2\). Choose \(M = {L}\) and let \(g_1, \hdots, g_k \in \mathbb{C}^{nk}\) be the first \(k\) eigenvectors of \(\mathbb{E}{L}\). Assume  $p-q \ge 40 \sqrt{pk\log(kn)/n}$. Then, with high probability,

\[
\sum_{i=2}^k \|f_i - \hat{g}_i\|^2 = \mathcal{O}\left(\frac{k}{p-q} \sqrt{\frac{pk\log(kn)}{n}} \right).
\]
\end{corollary}

Corollary~\ref{cor:sbm} together with Lemma~\ref{lem:kmeans} implies Spectral Clustering 
misclassifies vertices with a volume of at most $\mathcal{O}\left(\frac{k^2}{p-q} \sqrt{p(kn)\log(kn)} \right)$. We note that whilst this is not an improvement on what can be achieved with results from perturbation analysis such as the Davis-Kahan theorem \citep{davis1970rotation}, the original structure theorem achieves merely a constant bound for SBMs. 

In order to prove Corollary~\ref*{cor:sbm}, we require some prerequisite results.
\newline
\begin{thm} [\cite{chung2011spectra}] \label{thm:Chung bound}
    Let \(X_1, X_2,\hdots, X_m\) be independent random \(d \times d\) Hermitian matrices. Moreover, assume that \(\|X_j - \mathbb{E}(X_j) \| \leq M\) for all \(j\), and let \(\sigma^2  = \|\sum_{j=1}^m \mathbb{E}((X_j - \mathbb{E}(X_j))^2)\|\). Let \(X = \sum_{j=1}^m X_j\). Then, for any \(a>0\), it holds that:
    \[
    \mathbb{P}(\|X - \mathbb{E}(X)\| >a) < 2d \exp\left(\frac{-a^2}{2 \sigma^2 + 2Ma/3}\right)
    \]
\end{thm}

\begin{thm}[Courant-Fischer \citep{horn2012matrix}] \label{Courant-Fischer}
Let \( A \) be a Hermitian matrix in \( \mathbb{C}^{N \times N} \). The eigenvalues \( \lambda_1, \lambda_2, \ldots, \lambda_N \) of \( A \), arranged in non-increasing order (\( \lambda_1 \geq \lambda_2 \geq \ldots \geq \lambda_N \)), can be characterized by the Courant-Fischer min-max principle as follows:

For \( k = 1, 2, \ldots, N \):
\[
\lambda_k = \max_{\substack{S \subseteq \mathbb{C}^N \\ \dim(S) = k}} \min_{\substack{x \in S \\ x \neq 0}} \frac{\langle x, Ax \rangle}{\langle x, x \rangle}
\]
and
\[
\lambda_k = \min_{\substack{T \subseteq \mathbb{C}^N \\ \dim(T) = N-k+1}} \max_{\substack{x \in T \\ x \neq 0}} \frac{\langle x, Ax \rangle}{\langle x, x \rangle}
\]

where \( \langle \cdot, \cdot \rangle \) denotes the standard inner product in \( \mathbb{C}^N \).
\end{thm}

\begin{corollary}[Courant-Fischer Corollary] \label{CF Corollary}
Let \( A \) and \( B \) be Hermitian matrices in \( \mathbb{C}^{N \times N} \), and let \( \lambda_k(A) \) and \( \lambda_k(B) \) denote their \( k \)-th eigenvalues, respectively. Then
\[
|\lambda_k(A) - \lambda_k(B)| \leq \|A - B\|.
\]
\end{corollary}

In the following lemma, we apply Theorem~\ref{thm:Chung bound} to obtain a bound on  $\|L - \mathbb{E}(L)\|$ for SBMs.

\begin{lem} \label{lem:bound on L for SBMs}
Suppose that \(G \sim SBM(n,k,p,q)\) for some \(k \geq 2\). Let \(L\) be the Laplacian of \(G\). Then, with high probability,
\[
     \|L - \mathbb{E}(L)\| \leq 18 \sqrt{pkn \log (kn)}.\label{statement1}
\]
\end{lem}
\begin{proof}
    Firstly, let \[
    M_{uv} = (e_u e_u^T + e_v e_v^T) - (e_u e_v^T + e_v e_u^T),
    \]

    which has precisely four non-zero entries: $M_{uv}(u,u) = M_{uv}(v,v) = 1$ and $M_{uv}(u,v) = M_{uv}(v,u) = -1$. 
    Now we define the random matrix
\[X_{uv} = \begin{cases}
    M_{uv} \ \text{ if } u \sim v \\
    0 \ \text{ otherwise.}
\end{cases}\]
Notice that \(L= \sum_{(u,v) \in E} X_{uv}\).
Moreover, we have that \(\|X_{uv} - \mathbb{E}(X_{uv})\| \leq 2\).

By the identity
\[
 \mathbb{E}((X_{uv} - \mathbb{E}(X_{uv}))^2) = \mathbb{E}((X_{uv})^2) - (\mathbb{E}(X_{uv}))^2, 
 \]
 and the fact that $M_{uv}^2 = 2 M_{uv}$,
 we obtain that
\[
 \mathbb{E}((X_{uv})^2) - (\mathbb{E}(X_{uv}))^2  = \begin{cases}
     2p(1-p)M_{uv} & \ \text{if } u \text{ and } v \text{ are in the same block} \\
     2q(1-q)M_{uv} & \ \text{otherwise.}
 \end{cases}
\]
Summing over all \(u,v \in V\) and taking the norm, we obtain the following:
\[\sigma^2 = \lVert\sum_{\{u,v\} \in {V\choose 2}} \mathbb{E}((X_{uv})^2) - (\mathbb{E}(X_{uv}))^2\rVert \leq 4(p+(k-1)q)n.\]

Applying Theorem~\ref{thm:Chung bound} with  \(a = 18 \sqrt{pkn \log (kn)}\) yields the statement.


\end{proof}

The following result states the eigenvalues of the expected Laplacian in stochastic block models. Its proof is by elementary calculations.

\begin{lem} \label{eigenvalues expected laps SBMs}
Suppose that \(G \sim SBM(n,k,p,q)\) for some \(k \geq 2\). Let \(L = D - A\) be the Laplacian of \(G\) and let \(\mathcal{L} = I - D^{-1/2}AD^{-1/2}\) be the normalized Laplacian of \(G\). Then, the first eigenvalue of  \(\mathbb{E}(L)\) and \(\mathbb{E}(\mathcal{L})\) are both 0. The next \(k-1\) eigenvalues of \(\mathbb{E}(L)\) and \(\mathbb{E}(\mathcal{L})\) are \(knq\) and \(\frac{knq}{(n-1)p + (k-1)nq}\) respectively and the \((k+1)\)st eigenvalues are \(np + (k-1)nq\) and \(\frac{np + (k-1)nq}{(n-1)p + (k-1)nq}\).
\end{lem}

We are now ready to prove Corollary~\ref*{cor:sbm}.

\begin{proof}[Proof of Corollary~\ref*{cor:sbm}]
    Firstly, let \(\lambda_1(\mathbb{E}(L)) \le \dots \le \lambda_N(\mathbb{E}(L))\) denote the eigenvalues of \(\mathbb{E}(L)\) with corresponding orthonormal eigenvectors \(g_1, \hdots, g_k\). Let \(\gamma_i = \frac{g_i^* L g_i}{g_i^* g_i}\).
     Notice that \(g_1 = f_1 = \frac{1}{\sqrt{N}}1\). By Corollary~\ref*{cor:RemoveFirstEvec}, we have that
       \[\sum_{i=2}^k \|\hat{f}_i - g_i\|^2 \leq \frac{\sum_{i=2}^k (\gamma_i - \lambda_2)}{ \lambda_{k+1} - \lambda_2}.\]
    We will bound the numerator of the previous expression as follows:
     \begin{align*}
            |\gamma_i - \lambda_2| & = |\gamma_i - \lambda_i(\mathbb{E}(L)) + \lambda_i(\mathbb{E}(L)) - \lambda_2| \\
            & \leq |\gamma_i - \lambda_i(\mathbb{E}(L))| + |\lambda_i(\mathbb{E}(L)) - \lambda_2|.
        \end{align*}
    The first term can be bounded above as follows:
    \[
    |\gamma_i - \lambda_i(\mathbb{E}(L))| = \|g_i^*(L - \mathbb{E}(L)) g_i\| \leq \| L - \mathbb{E}(L) \|.
    \]
By Lemma~\ref{lem:bound on L for SBMs}, with high probability, \(\|L - \mathbb{E}(L)\| \leq 18\sqrt{(pkn)\log(kn)}\).
By Corollary \ref{CF Corollary}, for all \(i\), it holds that:
\begin{equation}
\label{eq:lambdai}
 |\lambda_{i} - \lambda_i(\mathbb{E}(L)| \leq \|L - \mathbb{E}(L)\| \leq 18\sqrt{(pkn)\log(kn)}.
\end{equation}
Thus, by Lemma \ref{eigenvalues expected laps SBMs},  for \(i = 1, \hdots, k\),
\[
\lambda_i \in \left[ knq - 18\sqrt{(pkn)\log(kn)}, knq + 18\sqrt{(pkn)\log(kn)} \right]
\]
which implies
\[
|\lambda_i(\mathbb{E}(L)) - \lambda_2| = |knq - \lambda_2| \leq 18\sqrt{(pkn)\log(kn)}
\]
and
\[
\gamma_i - \lambda_2 \leq 36 \sqrt{(pkn)\log(kn)}.
\]
On the other hand, Lemma~\ref{eigenvalues expected laps SBMs} and \eqref{eq:lambdai} imply that
\[
\lambda_{k+1} \in \left[ np + (k-1)nq - 18\sqrt{(pkn)\log(kn)}, np + (k-1)nq + 18\sqrt{(pkn)\log(kn)} \right]
\]
and, therefore, 
\begin{align*}
\lambda_{k+1} - \lambda_2 &\geq (np + (k-1)nq -18 \sqrt{(pkn)\log(kn)}) - (knq + 18\sqrt{(pkn)\log(kn)}) \\
&=  n(p-q) - 36 \sqrt{(pkn)\log(kn)}.
\end{align*}
Finally, by our assumption that $p-q \ge 40 \sqrt{pk\log(kn)/n}$, it holds that
\begin{align*}
\sum_{i=2}^k \|\hat{f}_i - g_i\|^2 &\leq
\frac{\sum_{i=2}^k (\gamma_i - \lambda_2)}{ \lambda_{k+1} - \lambda_2} \\ &\leq \frac{36(k-1)\sqrt{(pkn)\log(kn)}}{n(p-q) - 36\sqrt{(pkn)\log(kn)}} \\ 
&= \mathcal{O}\left(\frac{k}{p-q} \sqrt{\frac{pk\log(kn)}{n}} \right).
\end{align*}
\end{proof}



 \subsection{Section \ref*{sec:digraphs}}



\subsubsection{Proof of Lemma~\ref*{lem:PsiBounds}}
\begin{proof}
    Consider the rayleigh quotient of \(\mathcal{L}\) with some vector \(x \in \mathbb{C}^N\):
\begin{align*}
    \frac{x^* \mathcal{L} x}{x^* x}&  = \frac{x^* (I - D^{-1/2} A D^{-1/2}) x}{x^* x} \\
    & = \frac{y^* (D -  A) y}{y^* D y}  \ (\text{where } y = D^{-1/2}x) \\
    & = \frac{\sum_{u \rightarrow v} |y_u|^2 - \bar{y}_u y_v \omega_k - y_u \bar{y}_v \bar{\omega_k} + |y_v|^2}{\sum_{u \rightarrow v} |y_u|^2 + |y_v|^2} \\
    & = \frac{\sum_{u \rightarrow v} |y_u - \omega_k y_v|^2}{\sum_{u \rightarrow v} |y_u|^2 + |y_v|^2}
\end{align*}
where \(\omega_k = \exp(\frac{2 \pi i}{k})\).
Selecting \(\chi_{\mathcal{S}}\)
\[
\chi_{\mathcal{S}}(u) = \begin{cases}
    \sqrt{\frac{d(u)}{\vol(V)}} \cdot \mathrm{e}^{\frac{2\pi \mathrm{i} j}{k}} \ & \text{ for } u \in S_j, \\
    0 \ & \text{ otherwise.}
    \end{cases}
\]
Then the expression resolves to
\begin{align*}
    \frac{\chi_{\mathcal{S}}^* \mathcal{L} \chi_{\mathcal{S}}}{\chi_{\mathcal{S}}^* \chi_{\mathcal{S}}} & =  \frac{1}{\text{vol}(V)} \sum_{i=1}^k \sum_{j=1}^k 2w(S_i,S_j)(1 - \cos \left(\frac{(i-(j+1))2 \pi}{k}\right)) \\
    & = \frac{4}{\text{vol}(V)} \sum_{i=1}^k \sum_{j=1}^k w(S_i,S_j) \sin^2 \left( \frac{(i-(j+1)) \pi}{k} \right) \\
    & = \frac{4}{\text{vol}(V)} \sum_{(i,j) \notin C} w(S_i,S_j) \sin^2 \left( \frac{(i-(j+1)) \pi}{k} \right) \\
    & \leq \frac{4}{\text{vol}(V)} \sum_{(i,j) \notin C} w(S_i,S_j) \\
    &  = 4 \Psi(S_1,\hdots, S_k)
\end{align*}
The second equality follows from the identity \(1 - \cos(x) = 2\sin^2(x/2)\), the third equality follows as \( \sin^2 \left( \frac{(i-(j+1)) \pi}{k} \right) = 0\) for \((i,j) \in C\), and the inequality holds because \(\sin^2(x) \leq 1\).
\newline
Now to prove the lower bound, consider the following function:
\[
y(x) = \begin{cases}
    \frac{2}{\pi}(-x-\pi) \ &  \text{for } x \in [-\pi,-\pi/2)\\
    \frac{2}{\pi}x \ & \text{for } x \in [-\pi/2, \pi/2] \\
    \frac{2}{\pi}(\pi - x) \ & \text{for } x \in (\pi/2, \pi]
\end{cases}
\]
It can be easily verified that \(\sin^2(x) \geq y(x)^2\). Considering only \((i,j) \notin C\), the smallest values of \(\sin^2\left(\frac{(i-(j+1))\pi}{k}\right)\) are attained when \(i - (j+1) = \pm 1, \pm (k-1)\). The resulting value of \(y\left( \frac{(i-(j+1)) \pi}{k} \right)^2\) in all these instances is \(\frac{4}{k^2}\). The bound follows.
\end{proof}


\subsubsection{Proof of Lemma~\ref*{lem:EquivalenceLambda}}
\begin{proof}
Let us first assume $\Psi_k(\mathcal{G}) = 0$. Let $S$ be a $k$-way partitioning such that $\Psi(S_1,\dots,S_k)=0$. Then, by ~\ref{lem:PsiBounds}, $\chi_{\mathcal{S}}^* \mathcal{L}\chi_{\mathcal{S}} = 0$. This  implies the if direction.

For the reverse direction, let $y$ be an eigenvector of $\mathcal{L}$ with eigenvalue $0$, i.e., $y^* \mathcal{L}y = 0$. Then, as shown in the proof of ~\ref{lem:PsiBounds},
$y^* \mathcal{L}y = \sum_{u \rightarrow v} |y_u - \omega_k y_v|^2 = 0$.
By the weak connectivity of $\mathcal{G}$, this implies there exists $\alpha \in \mathbb{C} \setminus \{0\}$ such that, for any $u$, there exists $i\in \{0,\dots,k-1\}$ such that  $y_u = \alpha \omega_k^i$. 
Let us fix such $\alpha$ and let $S_1,\dots,S_k$ be a partition of $V$ defined as follows: $S_i = \{u \in V \colon y_u = \alpha \omega_k^{i-1}\}$ for $i=1,\dots,k$. We have that $\Psi(S_1,\dots,S_k)=0$, which implies $\Psi_k(\mathcal{G}) = 0$.
\end{proof}

\subsubsection{Proof of Theorem~\ref{thm:digraph}}
\begin{proof}
    The first statement follows from Theorem~\ref{thm:general} using \(k=1\) and \(g_1 = \chi_{\mathcal{S}}\). The second statement then follows from directly applying Lemma~\ref{lem:PsiBounds} to the first statement.
\end{proof}

\section{Omitted experiments}

\subsection{SBMs: Additional structure}
In addition to our SBM experiment in Section~\ref{sec:general}, we have a similar experiment for an SBM with 8 clusters where one pair of clusters have a higher affinity for each other. More precisely, \(P\) is an 8 by 8 matrix with \(0.5\) along its diagonal, \(P_{12} = P_{21} = 0.3\) and every other value is set to \(0.05\). The bounds for this SBM for a varied cluster size \(n\) can be seen in Figure~\ref{fig:sbm8cluster1pair}. 
\begin{figure}[h]
    \centering
        \includegraphics[width=0.5\linewidth]{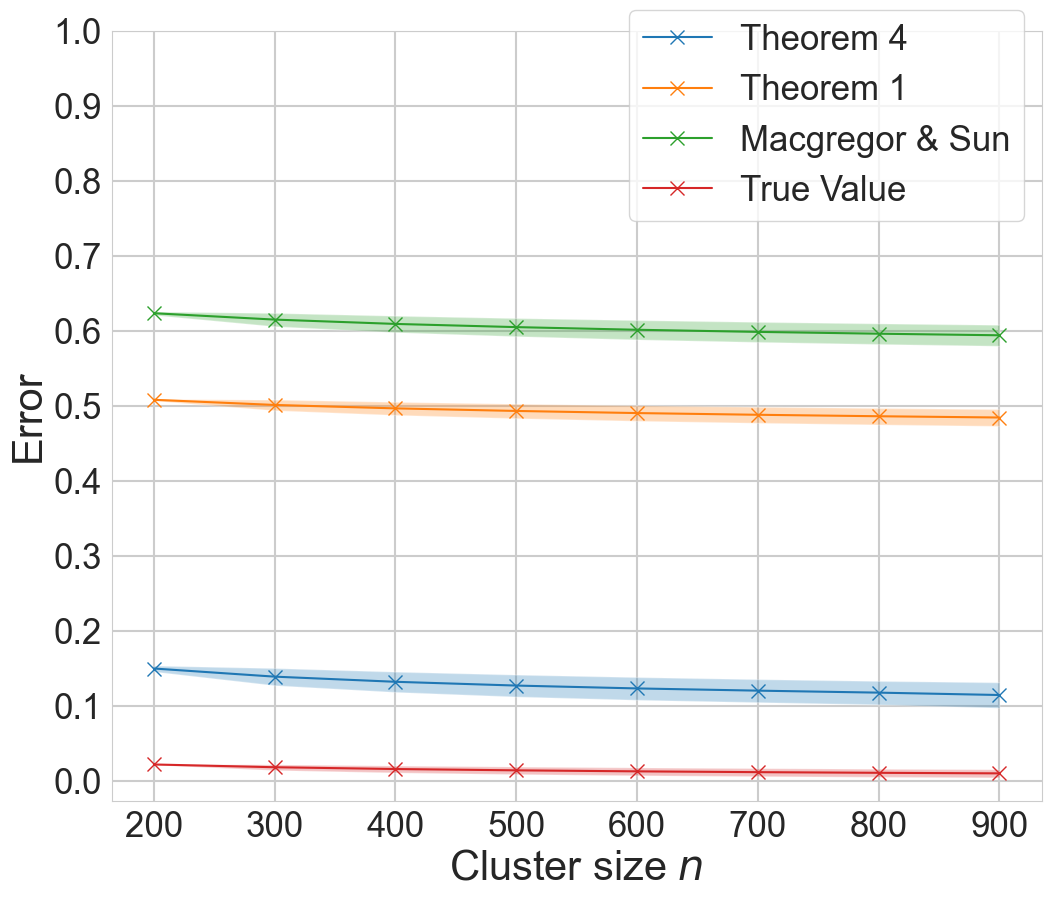}
        \caption{Stochastic block model with 8 blocks where one pair of blocks have a higher affinity to each other. Each point is an average over 10 realisations and the standard deviation is included as filled error bars.}
        \label{fig:sbm8cluster1pair}
\end{figure}

We also considered an SBM with multiple layers in its hierarchy of clusters. We took an SBM with 12 blocks that can be split into three sets of four. In each quartet, two pairs of the blocks have a higher affinity for each other. Precisely, this was the \(12 \times 12\) probability matrix \(P\) that was used:
\[
P = \begin{bmatrix}
Q & \bf{.05} & \bf{.05} \\
\bf{.05} & Q & \bf{.05} \\
\bf{.05} & \bf{.05} & Q \\
\end{bmatrix}
\]
Where 
\[Q = \begin{bmatrix}
& \cellcolor{yellow!90}.9 & \cellcolor{yellow!30}.7 & .2 & .2 & \\
& \cellcolor{yellow!30}.7 & \cellcolor{yellow!90}.9 & .2 & .2 & \\
& .2 & .2 & \cellcolor{yellow!90}.9 & \cellcolor{yellow!30}.7 & \\
& .2 & .2 & \cellcolor{yellow!30}.7 & \cellcolor{yellow!90}.9 & \\
\end{bmatrix}\]
and \(\bf{.05}\) is a \(4 \times 4\) matrix of \(0.05\). The bounds for this SBM for varied cluster size \(n\) can be seen in Figure~\ref{fig:sbm12cluster4quartets}.

\begin{figure}[h]
    \centering
    \includegraphics[width=0.5\linewidth]{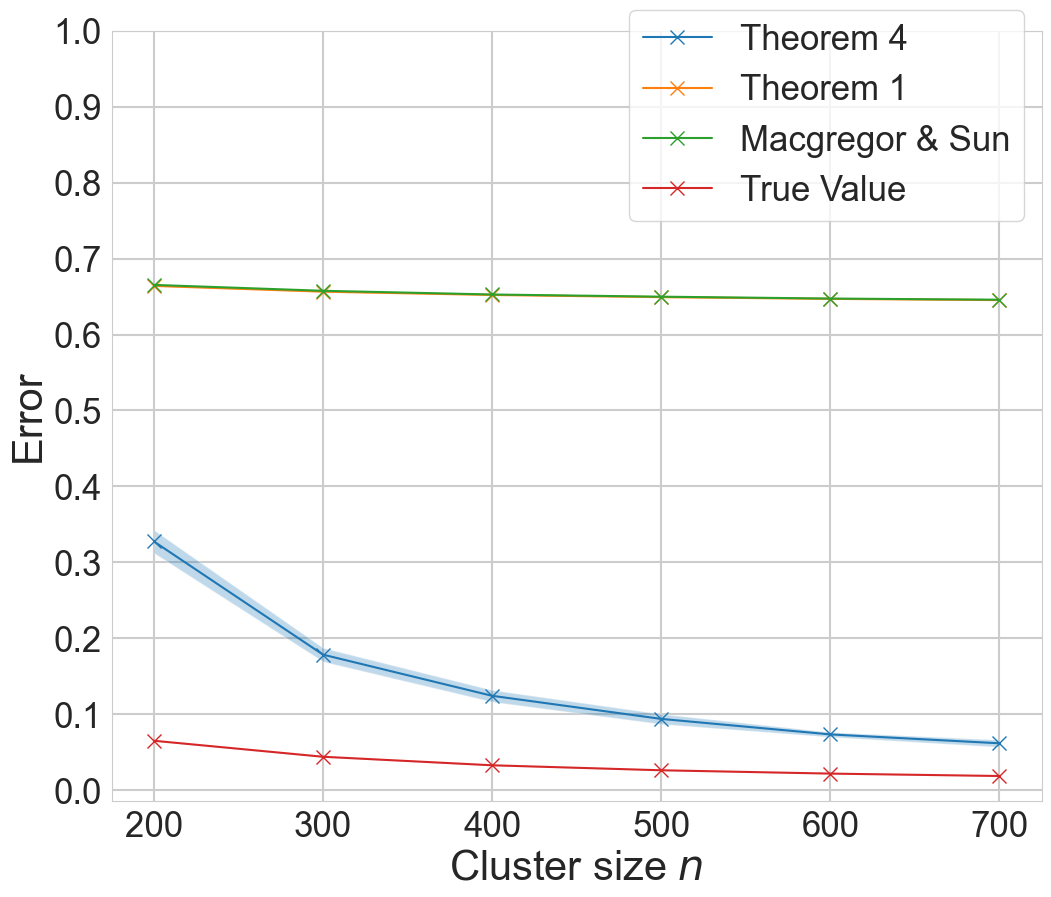}
    \caption{Stochastic block model with 12 blocks where three sets of four blocks have a higher affinity to blocks inside their quartet. Inside each quartet, two pairs of blocks have an even higher affinity to each other. Each point is an average over 5 realisations and the standard deviation is included as filled error bars.}
    \label{fig:sbm12cluster4quartets}
\end{figure}

In both experiments, we see that Theorem~\ref{thm:rec} greatly outperforms the other bounds. 

\subsection{SBMs: Close to the perfect recovery threshold}
We now consider our bounds for an SBM with 2 blocks with probability matrix
\[P = \begin{pmatrix}
    p & q \\
    q & p \\
\end{pmatrix}\]
where \(p = \frac{\alpha \log(N)}{N}, q = \frac{\beta \log(N)}{N},\) with \(\alpha\) and \(\beta\) taken at or close to the perfect recovery threshold \(\sqrt{\alpha} - \sqrt{\beta} \ge \sqrt{2}\). Details of this threshold and its significance can be found in \cite{hajek2016achieving}. We study the results of several experiments for different values of $\alpha$ and $\beta$.
In Figure~\ref{fig:ThresholdVaryingNBeta20LogScale}, we try \(\beta = 20\) with \(\alpha\) satisfying \(\sqrt{\alpha} - \sqrt{\beta} = 2\) and study the bounds for a varied cluster size \(n\). In Figure~\ref{fig:ThresholdVaryingNBeta1LogScale}, we try the same experiment with \(\beta=1\).

\begin{figure}[h!]
    \centering
    \begin{subfigure}[t]{0.48\linewidth}
        \centering
        \includegraphics[width=\textwidth]{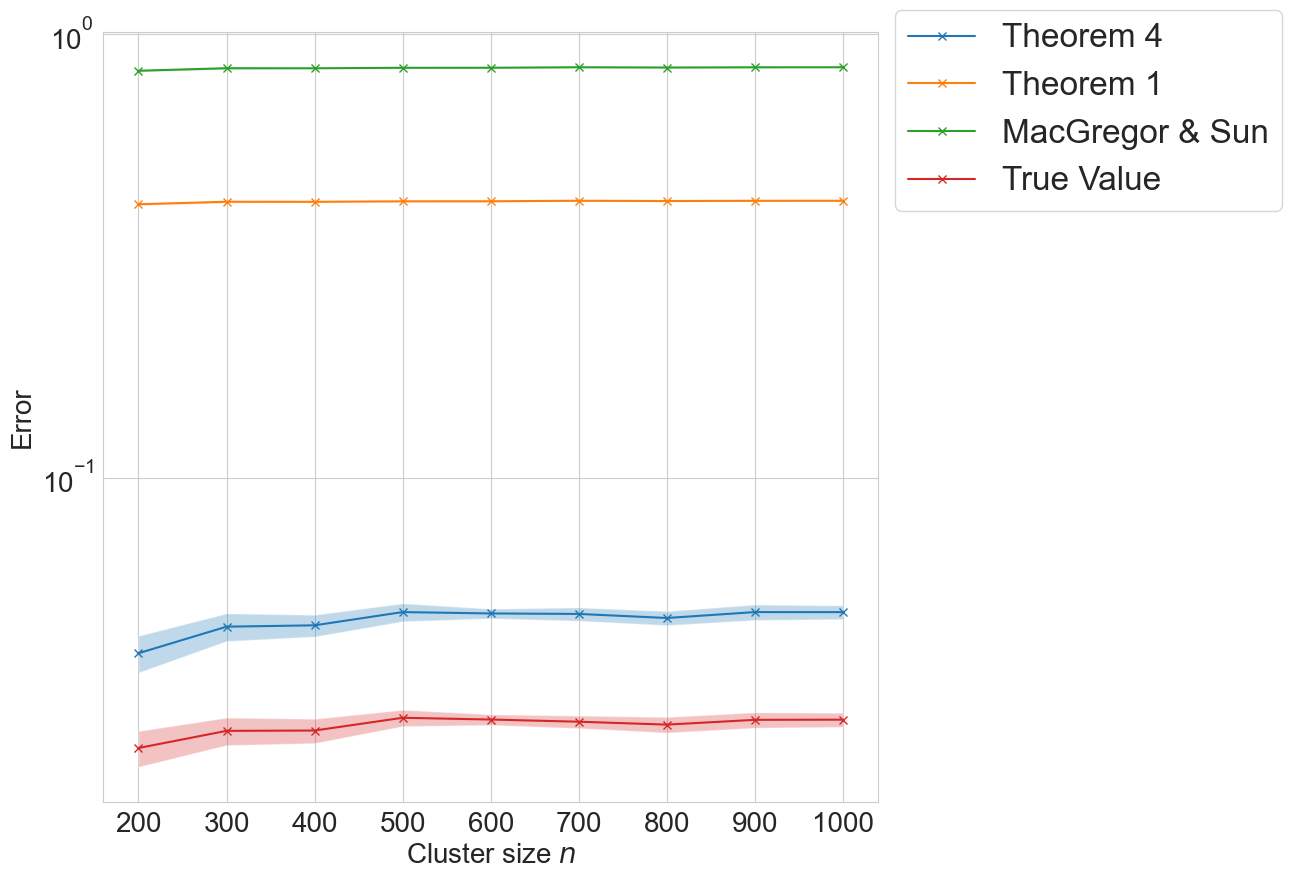}
        \vspace{0.01cm}
        \caption{\(\beta = 20\)}
        \label{fig:ThresholdVaryingNBeta20LogScale}
    \end{subfigure}
    ~
    \begin{subfigure}[t]{0.48\linewidth}
        \centering
        \includegraphics[width=\textwidth]{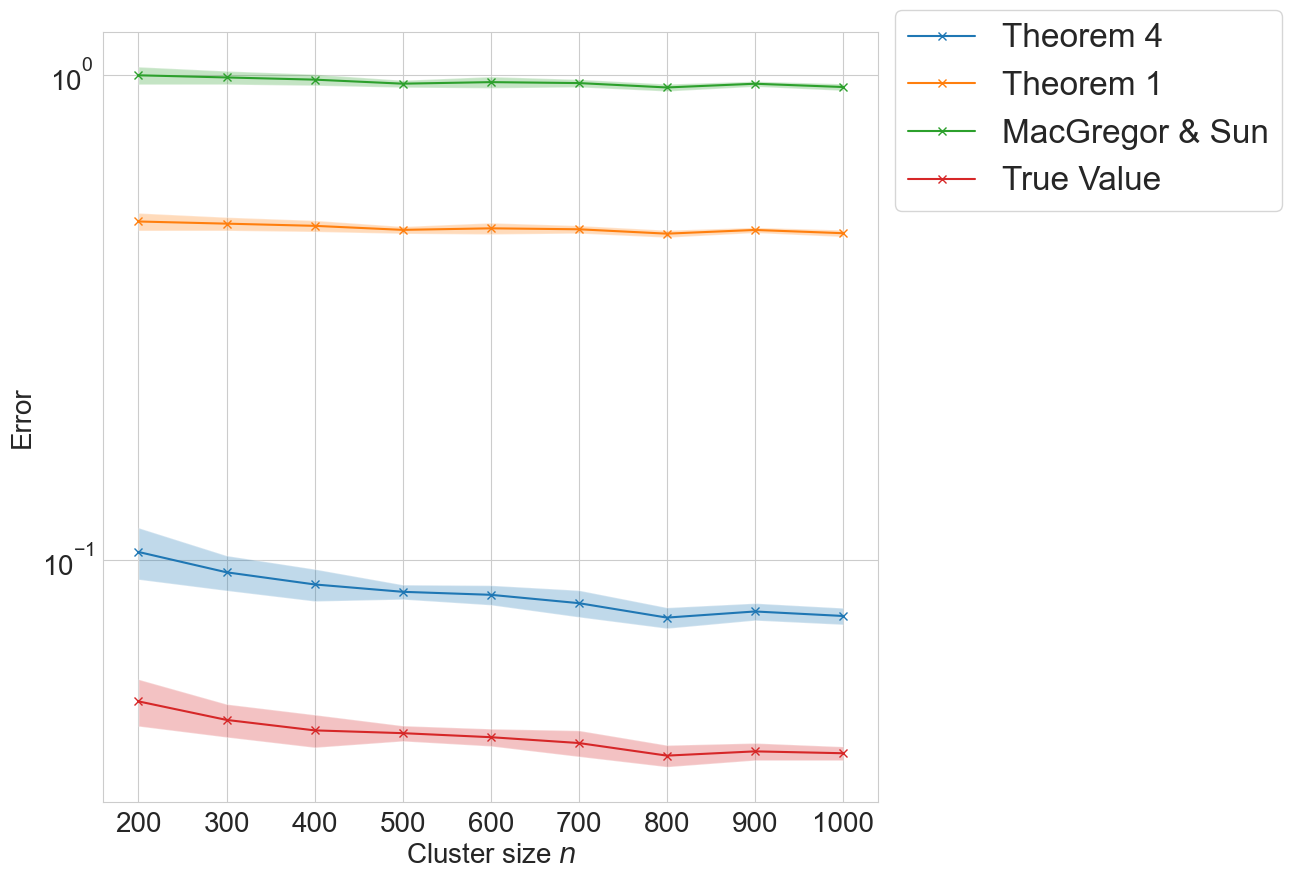}
        \vspace{0.01cm}
        \caption{\(\beta = 1\)}
        \label{fig:ThresholdVaryingNBeta1LogScale}
    \end{subfigure}
    \caption{\(\beta = 20\) (left) and \(\beta =1\) (right), \(\alpha = (\sqrt{\beta} + \sqrt{2})^2\), varying cluster size \(n\). An average is taken over 10 realisations for each value and the standard deviation is included as filled error bars.}
\end{figure}

The bounds do not vary much for our choices of \(n\), however we see that Theorem~\ref{thm:rec} provides a tighter bound on the true value than the other results.
We also considered varying \(\beta\) with \(\alpha\) satisfying \(\sqrt{\alpha} - \sqrt{\beta} = 2\) and cluster size fixed at \(n=500\). The results can be seen in Figure~\ref{fig:ThresholdVaryingBetaAtThresholdLogScale}. \

\begin{figure}[h]
    \centering
    \includegraphics[width=0.7\linewidth]{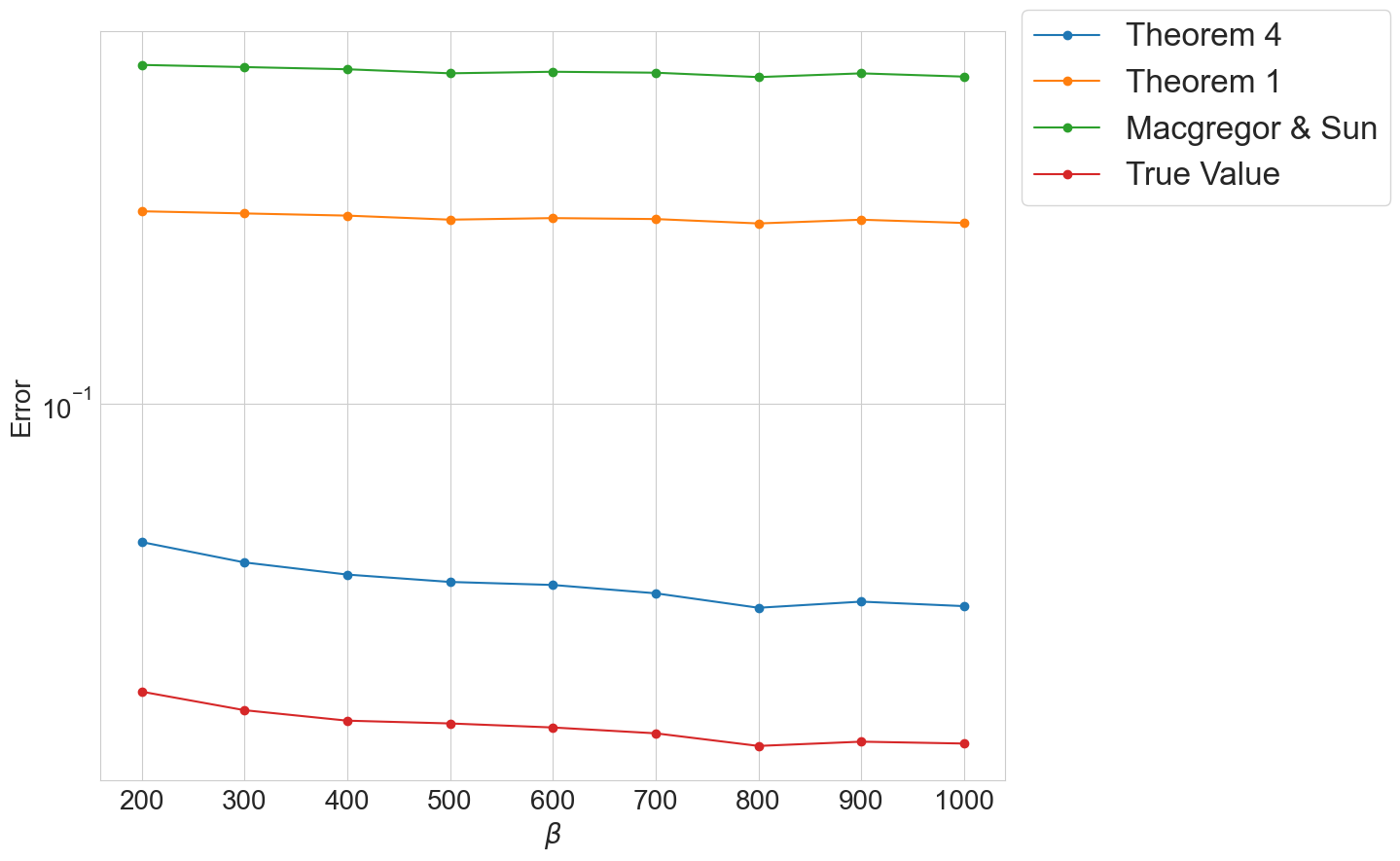}
    \caption{cluster size \(n = 500\); \(\beta\) is varied and \(\alpha\) is parametrised by \(\beta\) with \(\alpha = (\sqrt{\beta} + \sqrt{2})^2\). Each value is an average over 10 realisations and the standard deviation is included as filled error bars. }
    \label{fig:ThresholdVaryingBetaAtThresholdLogScale}
\end{figure}

We notice a similar behaviour in that the bounds do not change much at all as we vary our parameter \(\beta\). Again, the bound from Theorem~\ref{thm:rec} is much closer to the true value. Finally, we consider varying \(\beta\) for a fixed choice of \(\alpha\) (\(\alpha = 35\)). This will mean that as \(\beta\) is increased we will get closer to the threshold and eventually will be in violation of it. The results are pictured in Figure~\ref{fig:ThresholdVaryingBetaFixedAlphaAtThresholdLogScale}.

\begin{figure}[h]
    \centering
    \includegraphics[width=0.7\linewidth]{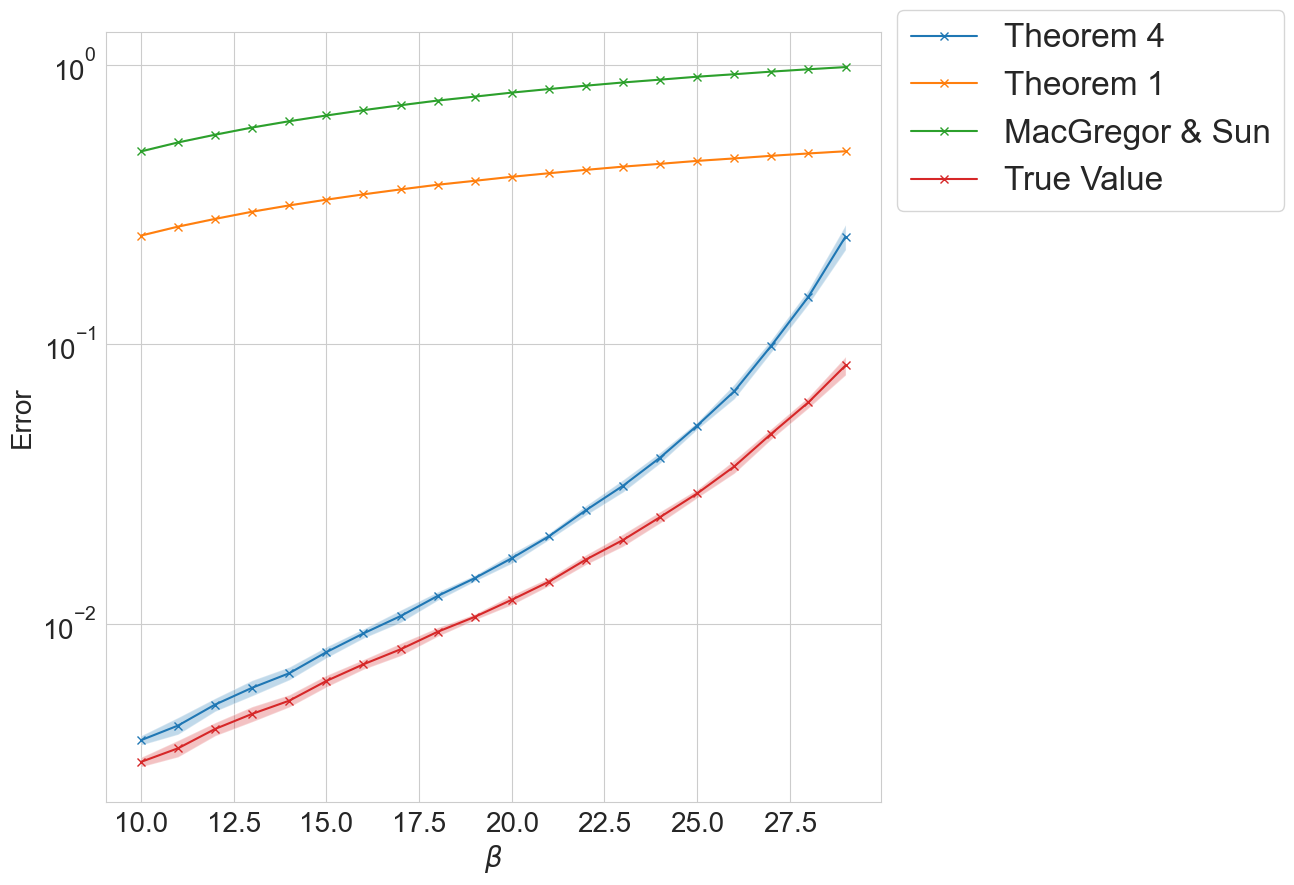}
    \caption{\(n = 500\); \(\beta\) is varied and \(\alpha\) is fixed at \(\alpha = 35\). Each value is an average over 10 realisations and the standard deviation is included as filled error bars. }
    \label{fig:ThresholdVaryingBetaFixedAlphaAtThresholdLogScale}
\end{figure}
In this experiment, we see that the true distance of the eigenvectors from the indicator vectors of the clusters grows with \(\beta\), and the bound from Theorem~\ref{thm:rec} stays relatively close to it. The other bounds are rather poor and increase slightly.

\subsection{DSBMs: Cyclic cluster structure}

Similar to our experiment in Section~\ref{sec:digraphs}, we consider a graph \(\mathcal{G} \sim \text{DSBM}(k,n,P,F)\) where \(k = 4\), \(n = 100\) and \(P\) and \(F\) are defined as follows:
\[
F = \begin{pmatrix}
    &\cellcolor{yellow!30}.5 & \cellcolor{yellow!60}1 & \cellcolor{yellow!30}.5 & \cellcolor{yellow!0}0& \\
    &\cellcolor{yellow!0}0 & \cellcolor{yellow!30}.5 & \cellcolor{yellow!60}1 & \cellcolor{yellow!30}.5& \\
    &\cellcolor{yellow!30}.5 & \cellcolor{yellow!0}0 & \cellcolor{yellow!30}.5 & \cellcolor{yellow!60}1&  \\
    &\cellcolor{yellow!60}1 & \cellcolor{yellow!30}.5 & \cellcolor{yellow!0}0 & \cellcolor{yellow!30}.5& \\
\end{pmatrix}, \ P = \begin{pmatrix}
    &\epsilon & \cellcolor{yellow!60}1 & \epsilon & \cellcolor{yellow!60}1 &\\
    &\cellcolor{yellow!60}1 & \epsilon & \cellcolor{yellow!60}1 & \epsilon &\\
    &\epsilon & \cellcolor{yellow!60}1 & \epsilon & \cellcolor{yellow!60}1 &\\
    &\cellcolor{yellow!60}1 & \epsilon & \cellcolor{yellow!60}1 & \epsilon&\\
\end{pmatrix}.
\]
As before, \(\epsilon\) is a noise parameter that we vary. This results in a graph which has a distinct cyclic cluster structure for small \(\epsilon\). The results of this experiment can be seen in Figure~\ref{fig:4Cycle}.

\begin{figure}[h]
    \centering
    \includegraphics[width=0.7\linewidth]{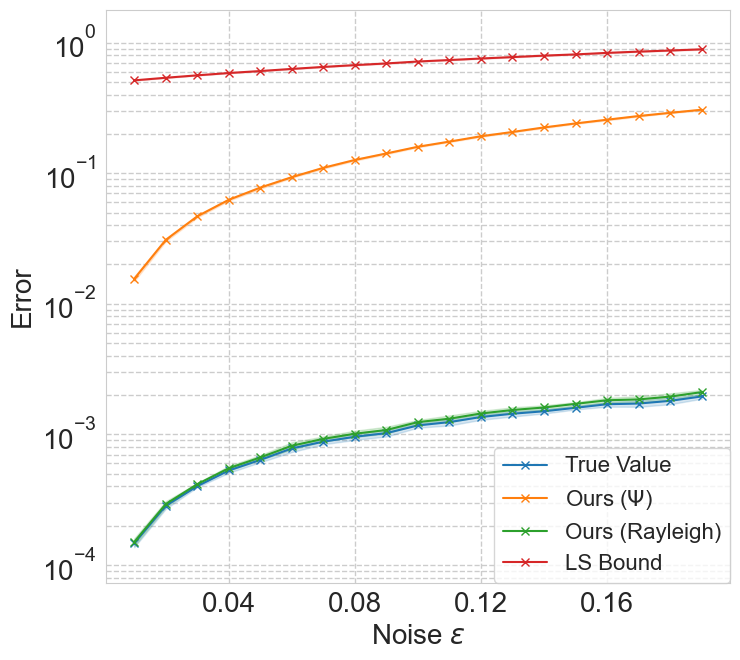}
    \caption{Comparison of the results given by the structure theorem of our paper (green for equation ~\ref{eq:ours_ray} and orange for equation ~\ref{eq:ours_psi}) and by Laenen and Sun (red) for two DSBMs at varying level of noise. The actual values are reported in blue. Each value is an average over 10 trials.}
    \label{fig:4Cycle}
\end{figure}

As one can see, our bounds significantly outperform the result of \cite{laenen2020higher}. In particular our bound from ~\ref{eq:ours_ray} hugs the true value very tightly for all noise levels.




\end{document}